\documentclass[twoside]{article}

\usepackage[accepted]{aistats2023}

\usepackage{fancyhdr}

\usepackage[utf8]{inputenc} 
\usepackage[T1]{fontenc}    
\usepackage{hyperref}       
\usepackage{url}            
\usepackage{booktabs}       
\usepackage{amsfonts}       
\usepackage{nicefrac}       
\usepackage{microtype}      
\usepackage{xcolor}         
\usepackage{placeins}

\usepackage{bbm}
\usepackage{mathtools} 
\usepackage{amsthm}
\usepackage{enumitem}
\usepackage{defs}
\usepackage{algorithm}
\usepackage{algpseudocode}

\usepackage[round]{natbib}

\begin{document}

\twocolumn[

\aistatstitle{Sparse Bayesian Optimization}
\aistatsauthor{ Sulin Liu$^{*}$ \And Qing Feng$^{*}$ \And David Eriksson$^{*}$ \And Benjamin Letham \And Eytan Bakshy}
\aistatsaddress{ Princeton University \And  Meta \And Meta \And Meta \And Meta }]

\begin{abstract}
Bayesian optimization (BO) is a powerful approach to sample-efficient optimization of black-box objective functions.
However, the application of BO to areas such as recommendation systems often requires taking the interpretability and simplicity of the configurations into consideration, a setting that has not been previously studied in the BO literature.
To make BO useful for this setting, we present several regularization-based approaches that allow us to discover sparse and more interpretable configurations.
We propose a novel differentiable relaxation based on homotopy continuation that makes it possible to target sparsity by working directly with $L_0$ regularization. 
We identify failure modes for regularized BO and develop a hyperparameter-free method, sparsity exploring Bayesian optimization (SEBO) that seeks to simultaneously maximize a target objective and sparsity.
SEBO and methods based on fixed regularization are evaluated on synthetic and real-world problems, and we show that we are able to efficiently optimize for sparsity.
\end{abstract}

\section{INTRODUCTION}
\label{sec:introduction}
Bayesian optimization (BO) is a technique for efficient global optimization that is used for parameter optimization across a wide range of applications, including robotics~\citep{Lizotte2007Automatic,Calandra2015a}, machine learning pipelines~\citep{hutter2011smac, snoek2012practical,turner2021bayesian}, internet systems~\citep{letham2019noisyei, feng2020cbo}, and chemistry~\citep{gomez2018chemical, Felton2021}. 
In many applications, including those just mentioned, it is preferable for the optimized parameters to be sparse. 
In this paper, we define \textit{sparsity} in Bayesian optimization to be the property where the majority of optimized parameters are close to the target parameters that one wishes to regularize towards. For example, the target parameters may be a zero-vector, where setting parameters to zero encourages removal of redundant system configurations. Alternatively, the target parameters may be the default system parameters (status quo), where sparsity favors the fewest modifications for consistency and robustness.
One reason to prefer sparsity is that it increases interpretability, a consideration that has recently attracted a great deal of attention in machine learning~\citep{doshi2017towards, rudin22}.
Interpretability is necessary for humans to be able to understand and evaluate the outputs of complex systems---the types of systems to which BO is often applied.
In policy optimization, sparsity of the control policy provides a natural way for human decision-makers to gain insight into the behavior of the system, and identify potential issues~\citep{UstunRu2016SLIM, HuRuSe2019}. 

Besides interpretability, sparsity can also be beneficial by producing systems that are easier to deploy and maintain, reducing the ``tech debt'' of machine learning systems~\citep{techdebt}. 
As an example, recommender systems are essential to many internet companies, including e-commerce platforms, streaming services, and social media sites~\citep{bobadilla2013recommender}. A typical recommendation process involves two stages, the retrieval and ranking stages~\citep{covington2016deep}. The parameters in the retrieval stage determine the amount of content to be fetched from various sets of candidate pools (\textit{sources}) representing different user interest taxonomies~\citep{wilhelm2018dpp}. Setting parameters to zero means deactivating these sources.
Sparse optimization can find solutions in which low quality sources are entirely turned off, thus simplifying the system and enabling faster development. Similarly in chemistry, a sparse solution may require fewer reagents and steps to synthesize a compound, which reduces experimentation overhead and accelerates the discovery of new compounds.
\begin{figure}[!ht]
\centering
    \begin{minipage}[c]{0.41\textwidth}
    \centering
    \includegraphics[width=0.8\textwidth]{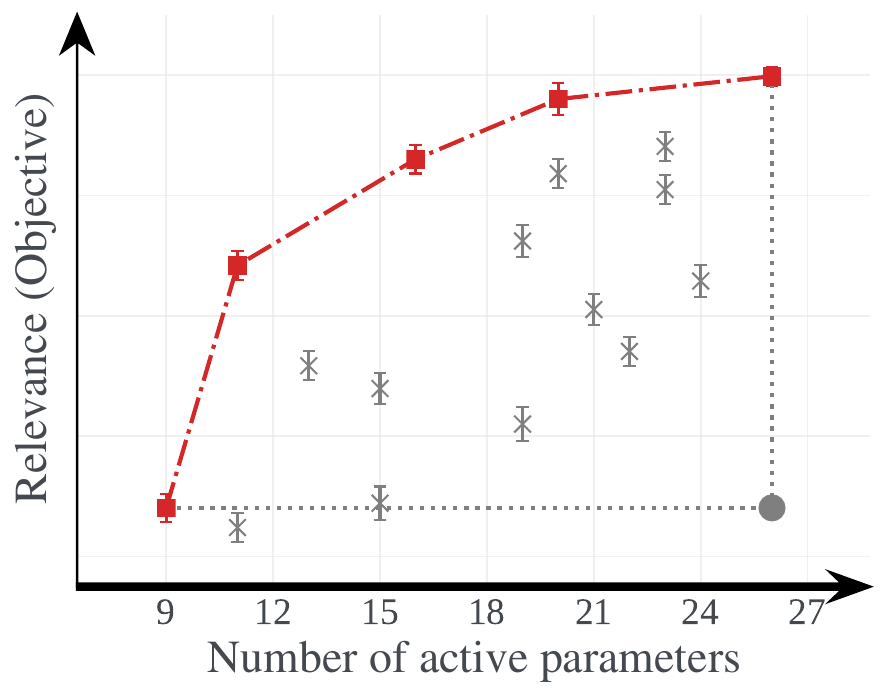}
    \caption{
    Objective and sparsity trade-offs for a real-world Internet experiment using SEBO. Points indicate recommender system configurations, where the x-axis corresponds to the number of active recommendation sources used, i.e. non-sparse parameters. Grey points indicate sub-optimal designs, while red points represent designs along the Pareto frontier found by SEBO. Decision-makers balance both system simplicity and performance when deciding which configuration to use.
        }
        \label{fig:sourcing_intro}
    \end{minipage}
\end{figure}

Sparsity in machine learning is often achieved via regularization, such as $L_1$ regularization used by the lasso~\citep{lasso}, the group norm penalty used by the group lasso~\citep{grouplasso}, and $L_0$ regularization which directly targets setting elements to zero~\citep{zhang2008multi}.
The purpose of regularization in machine learning is typically to limit overfitting and 
thus improve test accuracy by reducing generalization error~\citep{evgeniou02}. In our setting, sparsity is a separate goal; interpretable sparse configurations will generally not improve the optimization objective, and in fact, may come at some cost to other metrics. This can be seen in the sparsity-objective Pareto frontier shown in Fig.~\ref{fig:sourcing_intro} from a real-world recommender system sourcing experiment conducted at a large Internet firm. The Pareto frontier comprises all of the configurations that produce optimal trade-offs between sparsity and the optimization objective. In many real-world systems, decision makers are willing to trade some amount of objective in order to achieve a higher level of sparsity, because of the interpretability and simplicity benefits that come with sparsity. Thus, unlike a typical BO problem, the ``optimal" point per the decision maker will not necessarily be the one with best objective, but could be some other point on the sparsity-objective Pareto frontier that has more sparsity.

A central aspect of this work is to efficiently learn these trade-offs and offer practitioners a way to balance sparsity and other metrics.  
Sparsity in BO is an important topic that has not yet been addressed in the literature. 
Past work has used regularization in acquisition function optimization or modeling, but not for the purpose of sparsity in design parameters (see Section \ref{sec:background} for a review). 
Our work provides a thorough and broad treatment of sparsity in BO that fills in this gap. The main contributions of this paper are:
\begin{enumerate}[leftmargin=*]
    \setlength\itemsep{0em}
    \item We study different approaches for incorporating sparse regularization into BO, and provide negative theoretical results showing that previously studied forms of regularization can fail to optimize for certain levels of sparsity, regardless of the regularization coefficient.
    \item We draw connections between multi-objective BO and acquisition function regularization, and show how multi-objective BO can be used for automatic selection of the regularization coefficient. 
    We refer to this as the SEBO (``Sparsity Exploring Bayesian Optimization'') method.
    \item We develop a novel relaxation strategy for optimizing directly for $L_0$ sparsity, and show that it significantly outperforms the typical $L_1$ penalty in our context.
    \item We show that combining acquisition function regularization with sparse Gaussian process priors enables sparse optimization in high-dimensional spaces.
    \item We provide the first results on achieving sparsity via BO, in a range of synthetic functions and on three real-world tasks (in systems configuration and AutoML), showing that SEBO is the best approach for sparse BO. 
    We show the breadth of our method by using it to achieve different forms of sparsity such as feature-level and group sparsity.
    \item We provide a new high-dimensional benchmark problem designed to emulate trade-offs found in real-world recommender systems, and show how such systems benefit from increased sparsity.
\end{enumerate}

Section~\ref{sec:background} describes the necessary background and related work. Section~\ref{sec:methods} describes two natural approaches for incorporating sparse regularization into acquisition function optimization, both of which can fail to optimize for some levels of sparsity. 
Section~\ref{sec:moo} discusses a relationship between sparse BO and multi-objective BO, and describes how we can use methods from multi-objective BO to simultaneously optimize for all levels of sparsity. 
We describe how we optimize with $L_0$ regularization in Section~\ref{sec:l0_sparsity}.
We demonstrate the usefulness of our methods by applying them to a set of synthetic and real-world benchmarks in Section~\ref{sec:experiments}.
Finally, we discuss the results in Section~\ref{sec:discussion}.

\section{BACKGROUND AND RELATED WORK}
\label{sec:background}

\paragraph{Bayesian Optimization:}
\citet{shahriari2015review} provide a thorough review of BO. 
In short, the goal is to maximize a black-box function $f: \mathbb{R}^D \rightarrow \mathbb{R}$ over a compact set $\mathcal{B} \subset \mathbb{R}^D$, for simplicity taken as $[0, 1]^D$. 
We assume that $f$ is continuous and bounded on this domain.
At each iteration of optimization, $f$ is modeled with a Gaussian process (GP) given the function evaluations observed so far, producing the normally distributed posterior $f(\mathbf{x}) \sim \mathcal{N}(\mu(\mathbf{x}), \sigma^2(\mathbf{x}))$.
The location of the next function evaluation is selected by maximizing an acquisition function $\alpha(\mathbf{x}):=\mathbb{E}_f[u(\mathbf{x})]$ where $u$ is a utility function that defines the acquisition function.
Typical acquisition functions include expected improvement \citep[EI,][]{jones98} and upper confidence bound \citep[UCB,][]{srinivas10}. 
EI is given by
\begin{equation}\label{eq:EI}
    \alpha_{\textrm{EI}}(\mathbf{x}) = \mathbb{E}_{f} \left[ (f(\mathbf{x}) - f(\mathbf{x}^*))_+ \right],
\end{equation}
where $\mathbf{x}^*$ is the best point observed so far, and the acquisition function has a well-known analytic form when $f$ is a GP.
UCB is similarly computed directly from the marginal posterior,
\begin{equation}\label{eq:UCB}
    \alpha_{\textrm{UCB}}(\mathbf{x}) = \mu(\mathbf{x}) + \sqrt{\beta} \sigma(\mathbf{x}),
\end{equation}
where $\beta$ is a hyperparameter that controls the exploration-exploitation trade-off.
More recently, information-theoretic acquisition functions have been developed~\citep{hernandezlobato2014pes,wang2017maxvalue}.

\paragraph{Regularization in BO:} 
Regularization has been applied to acquisition function optimization, though not for the purpose of sparsity. 
\citet{shahriari2016unbounded} used regularization for unbounded BO, in which there are no bounds on the search space. 
They applied a form of $L_2$ regularization to the EI target value that penalized sampling points far from the initial center of the search space. 
\citet{gonzalez2016batch} used regularization for batch BO, where the penalty discouraged points from being chosen close to points that had already been selected for the batch.  The penalty term was multiplied with the original acquisition function value.

\paragraph{BO with Sparse Models:}
~\citet{saasbo} introduced the sparse axis-aligned subspaces (SAAS) function prior in which a structured sparse prior is induced over the inverse-squared kernel lengthscales $\{\rho_i\}_{i=1}^{d}$ to enable BO in high dimensions. 
The SAAS prior has the form $\tau \sim \mathcal{HC}(\alpha), \rho_{i} \sim \mathcal{HC}(\tau)$ where $\mathcal{HC}$ is the half-Cauchy distribution which concentrates at zero.
The goal of the SAAS prior is to turn off unimportant parameters by shrinking $\rho_i$ to zero, which avoids overfitting in high-dimensional spaces, thus enabling sample-efficient high-dimensional BO.
The global shrinkage parameter $\tau$ controls the overall sparsity: with more data, $\tau$ can be pushed to larger values, adapting the level of sparsity to the data as needed.

While sparsity in the GP model is different from the sparsity we seek here, we will show that combining the SAAS model with acquisition regularization is highly effective for sparse high-dimensional BO. 
By enforcing  regularization in the acquisition function, the parameters identified as unimportant will be set to their baseline values, generating simpler and more interpretable policies.
Other work has studied feature sparsity in GP regression but without considering sparsity in optimization~\citep{oh19, park21}.

\paragraph{Multi-Objective BO:}
Multi-objective BO is used when there are several (often competing) objectives $f_1, \ldots, f_m$ and we wish to recover the Pareto frontier of non-dominated configurations.
A classic method is ParEGO, which applies the standard single-objective EI acquisition function to a random scalarization of the objectives~\citep{parego}. 
Many types of scalarizations have been developed for transforming multi-objective optimization (MOO) problems into single-objective problems~\citep{Ehrgott05}. 
Recent work on multi-objective BO has focused on developing acquisition functions that explicitly target increasing the hypervolume of the known Pareto frontier. 
Acquisition functions in this class, such as Expected Hypervolume Improvement (EHVI), are considered state-of-the-art for multi-objective BO~\citep{yang2019, daulton2020ehvi, daulton2021nehvi}.

\section{ACQUISITION FUNCTION REGULARIZATION}
\label{sec:methods}

\subsection{External Regularization}\label{sec:external}
We use a regularization term $\xi(\mathbf{x})$ to model \textit{sparsity}, which may be an $L_0$ quasinorm to target feature-level sparsity, $\xi(\mathbf{x}) = \|\mathbf{x} - \mathbf{x}^s\|_0$, or can be adjusted for different forms of sparsity such as group sparsity. Here $\mathbf{x}^s$ represents the target point that the decision maker wishes to drive the solution towards, e.g., a zero-vector or the current default parameters (status quo). For our analysis of regularization, we will assume that $\mathbf{x}^s$ is the unique global minimum of $\xi(\mathbf{x})$. 

A straightforward approach for adding regularization is to simply add a regularization penalty directly to the acquisition function. 
This parallels regularized regression techniques like ridge regression and the lasso. 
Given a penalty term $\xi(\mathbf{x})$, we then maximize 
\begin{equation}\label{eq:ext}
\alpha_\textrm{ER}(\mathbf{x} ; \lambda) = \alpha(\mathbf{x}) - \lambda \xi(\mathbf{x})
\end{equation}
to select the next point for evaluation. We refer to this approach as \textit{external regularization} (ER). 
EI with external regularization is:
\begin{equation}
    \label{eq:EI_ext}
    \alpha_{\textrm{EI-ER}}(\mathbf{x} ; \lambda) = \mathbb{E}_{f} \left[ (f(\mathbf{x}) - f(\mathbf{x}^*))_+ \right] - \lambda \xi(\mathbf{x}).
\end{equation}

The regularization coefficient $\lambda$ must be set, just as with classic regularized regression.
This formulation separates the explore/exploit value of a point, in $\alpha$, from its sparsity value, in $\xi$. 
This can perform poorly, because there is necessarily interaction between these two notions of value. 
We provide a negative result showing that external regularization cannot capture certain levels of sparsity.

\begin{proposition}\label{thm:ext}
    Suppose $\alpha(\mathbf{x}) = 0$  for every $\mathbf{x}$ where $\xi(\mathbf{x}) \leq \theta$. 
    Then, for any value of $\lambda > 0$, every maximizer of $\alpha_\textrm{ER}(\mathbf{x} ; \lambda)$ will satisfy $\xi(\mathbf{x}) > \theta$, or will equal $\mathbf{x}^s$.
\end{proposition}
This result is shown in Appendix \ref{sec:proofs}, which also describes how this setting is easily encountered in practice when there is a trade-off between objective and sparsity, as in Fig. \ref{fig:sourcing_intro}. 
Empirically, Proposition \ref{thm:ext} means that once a good non-sparse point has been found, sparse points will not be selected by the regularized acquisition function, regardless of how $\lambda$ is tuned. 
Increasing $\lambda$ will change the maximum of the regularized acquisition function from a non-sparse point directly to the trivial solution of $\mathbf{x}^s$, skipping all levels of sparsity in between. The acquisition function has no way of selecting sparse points that improve over other points with a similar level of sparsity.

\subsection{Internal Regularization}\label{sec:internal}
An alternative approach for adding regularization to the acquisition optimization is to add it directly to the objective function. 
In this approach, instead of using the posterior of $f$ to compute the acquisition function, we compute the acquisition for the posterior of a regularized function:
\begin{equation}\label{eq:linscal}
    g(\mathbf{x}; \lambda) = f(\mathbf{x}) - \lambda \xi(\mathbf{x}).
\end{equation}
We refer to this as \textit{internal regularization} (IR). 
The goal of the acquisition function is then to maximize $g$, which can be made to have a sparse maximizer by appropriately setting $\lambda$. 
With internal regularization, EI becomes
\begin{equation}
\begin{aligned}
    \alpha_{\textrm{EI-IR}}(\mathbf{x} ; \lambda) 
    {} &= \mathbb{E}_{f} \left[ (g(\mathbf{x}) - g(\mathbf{x}^*))_+ \right] \\ 
    &= \mathbb{E}_{f} \left[ (f(\mathbf{x}) - f(\mathbf{x}^*) - \lambda (\xi(\mathbf{x}) - \xi(\mathbf{x}^*)))_+ \right] \label{eq:EI_int}
\end{aligned}
\end{equation}
where $\mathbf{x}^*$ is now the incumbent-best of $g$, not of $f$.
The difference between external and internal regularization depends on the acquisition function. 
It is easy to see that for the UCB acquisition of (\ref{eq:UCB}), they are identical. 
For EI they are not, as seen by comparing (\ref{eq:EI_ext}) and (\ref{eq:EI_int}). 
For EI, internal regularization avoids some of the issues of external regularization by incorporating sparsity directly into the assessment of improvement. 
In (\ref{eq:EI_int}), improvement is measured both in terms of increase of objective and increase in sparsity, and it is measured with respect to an incumbent best that has incorporated the sparsity penalty. 
However, internal regularization can also be incapable of recovering points at every level of sparsity, as we will show now. 
For this result, we are interested in the optimal objective value as a function of sparsity level:
\begin{equation}\label{eq:h}
    h(\theta) = \max_{\mathbf{x} \in \mathcal{B}} f(\mathbf{x}) \textrm{ subject to } \xi(\mathbf{x}) = \theta.
\end{equation}
A trade-off between sparsity and objective would result in $h(\theta)$ increasing with $\theta$, though it need not be strictly increasing. 
We now give the negative result for internal regularization, see Appendix \ref{sec:proofs} for details.
\begin{proposition}\label{thm:int}
    For any $\theta$ in the interior of an interval where $h$ is strictly convex, there is no maximizer of (\ref{eq:linscal}) with $\xi(\mathbf{x}) = \theta$, for any $\lambda > 0$.
\end{proposition}
This result shows that internal regularization can only hope to recover optimal points at all sparsity levels if $h$ is concave on its entire domain. 
This is a strong condition, one unlikely to hold for the types of functions typically of interest in BO, even with simple regularizers. 
Note that this result is independent of the choice of $\lambda$ and the acquisition function used. 
If the desired level of sparsity happens to lie within a region where $h$ is strictly convex, internal regularization can be expected to fail to find the optimum. 
Fig.~\ref{fig:ir_res_plot} in Appendix \ref{sec:proofs} shows an illustration of this result, in a problem where $h$ has a region of strict convexity.

We will see in the empirical results that internal regularization performs better than external regularization, though, consistent with Proposition \ref{thm:int}, can fail to cover the entire objective vs. sparsity trade-off and so neither is the recommended approach for sparse BO. 
In this paper we focus on EI, but both forms of regularization can be applied to any acquisition function, including entropy search methods. 
In entropy search, the acquisition function evaluates points according to their information gain with respect to the current belief about the location or function value of the optimum. 
The information gain will thus depend on the level of sparsity in a similar way as with EI, and so external and internal regularization have similar considerations.

\section{MULTI-OBJECTIVE OPTIMIZATION}
\label{sec:moo}
There are two fundamental challenges with both of the regularization approaches developed in Section \ref{sec:methods}. 
The first is that they both have a regularization coefficient $\lambda$ that must be set. 
In a regression setting, the regularization coefficient is usually set to maximize cross-validation accuracy through hyperparameter optimization, often using grid search or BO~\citep{snoek2012practical}. 
In sparse BO, if there is a known desired level of sparsity, $\lambda$ can be swept in each iteration of optimization to find a value that produces candidates with the desired level of sparsity. 
This significantly increases the overhead of BO by requiring hyperparameter optimization as part of every acquisition optimization.
Furthermore, in real applications the desired level of sparsity is typically not known \textit{a priori}.

When there is a trade-off between interpretability and system performance, the desired level of interpretability will depend on what that trade-off looks like. 
In practice, we thus wish to identify the best-achievable objective at any particular level of sparsity. 
The second challenge is that, per the results of Propositions \ref{thm:ext} and \ref{thm:int}, we may not be able to identify the entire objective vs. sparsity trade-off, no matter how $\lambda$ is swept. 
Depending on the problem, it may be that the sparsity levels of interest cannot be explored via either regularization strategy. 
Both of these challenges can be addressed by viewing sparse BO from the lens of multi-objective BO.

\subsection{Sparse BO as Multi-Objective BO}
In this section we introduce the Sparsity Exploring Bayesian Optimization method (SEBO), which takes a multi-objective approach to sparse BO. Rather than considering $\xi$ as a penalty applied to the objective, we consider $f$ and $-\xi$ to each be objectives that we wish to maximize.

First, we note the following connection between internal regularization and multi-objective BO.

\begin{remark}
Internal regularization can be viewed as a linear scalarization of the two objectives $f$ and $-\xi$, with $\lambda$ the weight. Linear scalarizations are commonly used in MOO~\citep{marler10}---see Appendix \ref{sec:SuppParEGO} for more discussion of the connection between internal regularization and the ParEGO method for multi-objective BO.
\end{remark}

Casting sparse BO as MOO of the objective and sparsity has several advantages. 
It provides a solution for setting the regularization coefficient $\lambda$, since we can use methods from multi-objective BO to optimally balance improvements in $f$ and $\xi$ with the goal of exploring the Pareto frontier. 
We can use powerful approaches such as EHVI to select points that maximize performance for all levels of sparsity, or equivalently, maximize sparsity for all levels of performance, explicitly optimizing for the entire regularization path.
The goal of multi-objective BO is to identify the optimum for every level of sparsity, which enables decision makers to make an informed trade-off between interpretability and other considerations of system performance. State-of-the-art MOO methods also avoid the issues of Propositions \ref{thm:ext} and \ref{thm:int} and are able to explore the entire Pareto front.

In our experiments, we use the EHVI acquisition function. Here, the hypervolume improvement is defined with respect to a worst-case reference point $\mathbf{r} = [r_f, r_\xi]$,  which can be set to estimates for the minimum and maximum values of $f$ and $\xi$ respectively. 
Given a set of observations $X^{\textrm{obs}} = \{\mathbf{x}^1, \ldots, \mathbf{x}^n\}$, the Pareto hypervolume of is defined as
\begin{equation*}
    V(X^{\textrm{obs}}) = \lambda_M \left( \bigcup_{i=1}^n \left( [r_f, r_\xi] \times [f(\mathbf{x}^i), \xi(\mathbf{x}^i)] \right)\right),
\end{equation*}
where $\lambda_M$ denotes the Lebesgue measure. 
The expected hypervolume improvement is computed as
\begin{equation}\label{eq:EHVI}
    \alpha_{\textrm{SEBO}}(\mathbf{x}) = \mathbb{E}_{f} \left[ V(X^{\textrm{obs}} \cup \{\mathbf{x}\}) - V(X^{\textrm{obs}}) \right].
\end{equation}
This acquisition function is hyperparameter-free, and, as we will see, is highly effective for sparse BO. In the experiments, we standardize the objectives when calculating the hypervolume. It is also possible to weight objectives differently to encourage greater exploration of sparse or high-performing solutions.
We refer to the resulting method as SEBO, and explore its performance in combination with the $L_0$ sparse regularization, described next. 
The SEBO-$L_0$ algorithm is shown in Appendix~\ref{subsec:supp-sebo}.  

\section{ACQUISITION FUNCTIONS WITH \texorpdfstring{$L_0$}{L0} SPARSITY}
\label{sec:l0_sparsity}
Our primary focus is $L_0$ sparsity, which comes with the challenge that the $L_0$ quasi-norm is discontinuous, making the resulting acquisition function challenging to optimize.
We will follow the idea of homotopy continuation, which has been successfully applied to, for instance, solving non-linear systems of equations and numerical bifurcation analysis~\citep{allgower2012numerical}.

The main idea is to define a homotopy $H(\mathbf{x}, a)$, where $H(\mathbf{x}, a_{\text{start}})$ corresponds to a problem that is easy to solve and $H(\mathbf{x}, a_{\text{end}})$ corresponds to the target problem.
In particular, for $a > 0$ we define $H(\mathbf{x}, a) = \mathbb{E}_f[u([f(\mathbf{x}), \varphi_a(\mathbf{x})])]$ where $\varphi_a(\mathbf{x}) := D - \sum_{i=1}^D  \exp\left(-0.5\,(\mathbf{x}_i/a)^2\right) \approx \|\mathbf{x}\|_0$ and $\mathbf{x} \in \mathbb{R}^D$. 
Under the assumption that the utility function $u(x)$ defined in Sec.~\ref{sec:background} is continuous, we have $\lim_{a \to 0^+} H(x, a) = \mathbb{E}_f[u([f(x), \|\mathbf{x}\|_0])]$, which corresponds to the original acquisition function with the $L_0$ quasi-norm.

While it may be tempting to set $a$ to a small value, e.g., $a=10^{-3}$, and optimize the acquisition function directly, this will not work well as the gradient of the homotopy is (numerically) zero almost everywhere in the domain.
On the other hand, setting $a$ to a large value, e.g., $a=1$ will make it much easier to optimize the acquisition function, but also result in a poor approximation of the true acquisition function that will likely not yield sparse solutions.
In order to optimize the acquisition function, we will start at some value $a_{\text{start}}$ large enough to make the acquisition function easy to optimize and slowly decrease $a$ towards $a_{\text{end}} = 0$.
Each time we change $a$ we re-optimize the acquisition function starting from the best solution found for the previous value of $a$.\footnote{This may appear similar to the idea of learning rate annealing. However, rather than decreasing a hyperparameter of the optimizer, we solve a sequence of optimization problems that approaches the true problem.}
This idea is illustrated in Fig.~\ref{fig:homotopy} where we plot snapshots of $H(\mathbf{x}, a)$ for a few values of $a$ as well as show the resulting continuous homotopy path.

\vspace{1em}
\begin{figure}[!ht]
    \centering
    \includegraphics[width=0.4\textwidth]{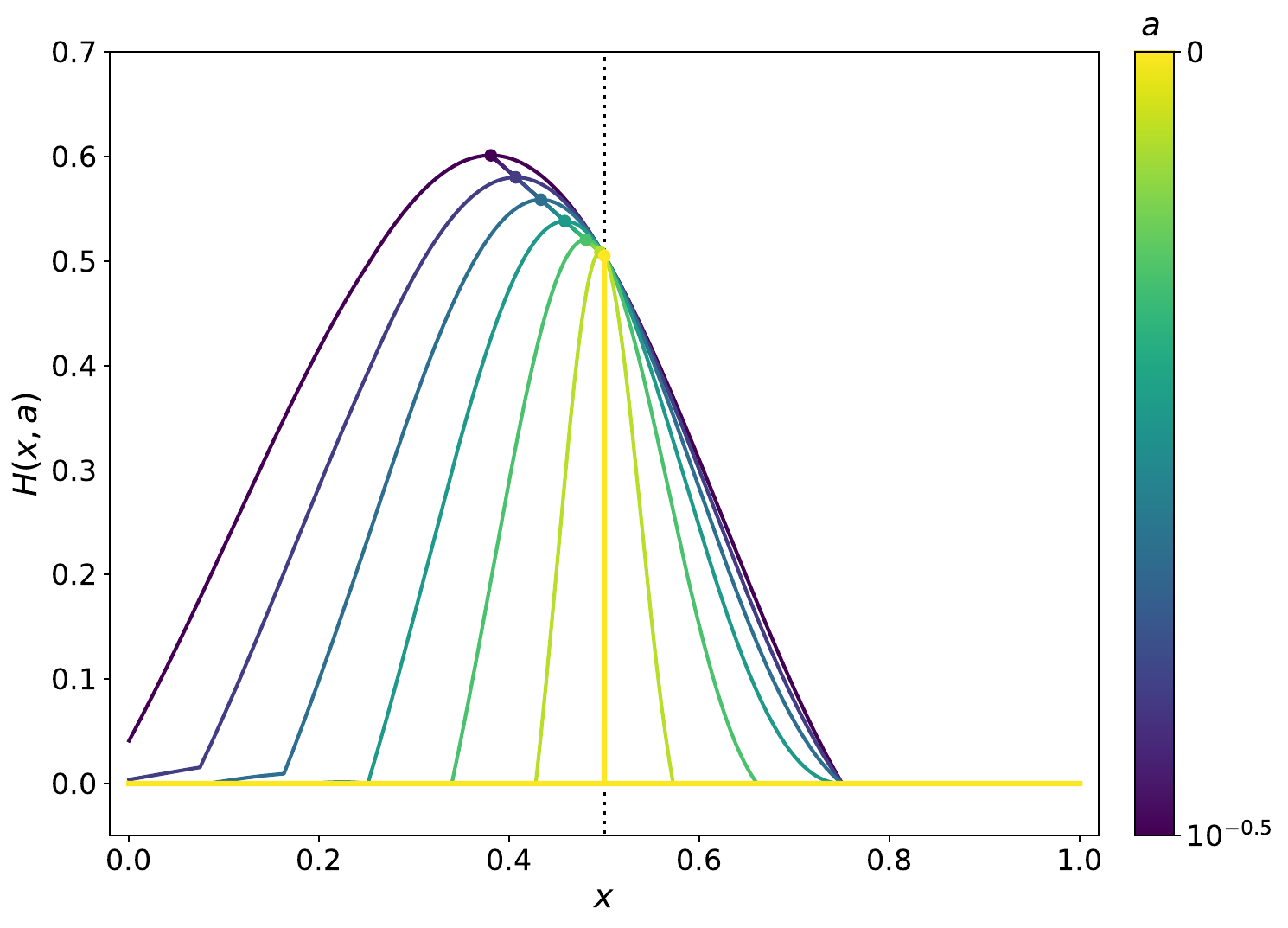}
    \caption{
        Consider the 1D problem of using SEBO to optimize $f(x)=-x^2$ with an $L_0$ penalty $\xi(x) = \|x - 0.5\|_0$. 
        Assume $X^{\textrm{obs}} = \{0, 0.25, 0.75, 1.0\}$ have already been evaluated and we want to optimize SEBO to generate the next candidate. 
        The global optimum of the acquisition function is given by the sparse point $x=0.5$. 
        We show that optimizing the acquisition function along the continuous homotopy path starting at $a_{\text{start}}=10^{-0.5}$ allows us to eventually uncover find the true optimum of $x=0.5$. 
    }
    \label{fig:homotopy}
\end{figure}

\section{EXPERIMENTS}
\label{sec:experiments}
We evaluate EI-IR, EI-ER and SEBO on two synthetic and three real-world problems with a focus on high-dimensional problems. Note that SEBO can be used for low-dimensional problems as well. Additional details are included in Appendix \ref{appendix:exp}. SEBO also naturally extends to multi-objective BO problems, and our code release supports that. We focus on single-objective problems to visualize and understand 2D Pareto frontiers, which are  difficult to visualize in higher dimensions. We show the results using $L_0$ regularization for most problems except for the last problem, where the group lasso is used to demonstrate that the methods can be applied to recover different forms of sparsity, such as group sparsity. 
In addition, we provide an ablation study that demonstrates the importance of using $L_0$ regularization by comparing it to $L_1$ regularization. We show in an ablation study that the homotopy continuation approach from Section~\ref{sec:l0_sparsity} is crucial for effective $L_0$ regularization.

\paragraph{Experimental setup:}
Our experiments all have high-dimensional parameter spaces, so we use the SAAS model when optimizing with ER, IR, and SEBO.
We compare performance to quasi-random search (Sobol), BO with a standard ARD Mat\'{e}rn-$5/2$ kernel and the EI acquisition function (GPEI), and SAASBO.
For the SAAS model, we use the same hyperparameters as suggested by~\citet{saasbo} and use the No-U-Turn (NUTS) sampler for model inference. 
The acquisition function is computed by averaging over the MCMC samples.
We always scale the domain to be the unit hypercube $[0, 1]^D$ and standardize the objective to have mean $0$ and variance $1$ before fitting the GP model. 

For the homotopy continuation approach described in Sec.~\ref{sec:l0_sparsity}, we discretize the range of $a$ to use $30$ values starting from $a_{\text{start}} = 10^{-0.5}$, see Appendix \ref{appendix:exp} for more details.
Fig.~\ref{fig:ablation_ell_starting} shows that SEBO is not sensitive to the choice of $a_{\text{start}}$.
We use a deterministic model for sparsity when using it as an objective.
The figures show the mean results across replications ($10$ replications for the adaptive bitrate simulation (ABR) problem and $20$ for all other experiments), and the error bars correspond to $2$ standard errors.
All experiments were run on a Tesla V100 SXM2 GPU (16GB RAM).
Code for replicating the methods and benchmark experiments in this work is available at \url{https://github.com/facebookresearch/SparseBO}.

\paragraph{Evaluation plots:}\label{plots}
We evaluate optimization performance in terms of the trade-off between the objective and sparsity. To compare the trade-offs, we show the resulting Pareto frontier by treating sparsity as a separate objective, e.g., Fig.~\ref{fig:synthetic} (Right) and Fig.~\ref{fig:real_world}. In particular, for each level of sparsity (active dimensions), we plot the best value found using \emph{at most} that number of non-sparse components.
We also show hypervolume traces in the Appendix~\ref{subsec: Supp-hv}. 
In cases where a method is unable to find at least one configuration for a given level of sparsity we assign replications an imputed function value corresponding to the worst label shown on the y-axis. 
For the synthetic problems where the true active dimensions and optima are known, we plot simple regret for a fixed level of sparsity, e.g., in Fig.~\ref{fig:synthetic} (Left, Middle).

\begin{figure*}[!ht]
    \centering
    \includegraphics[width=0.975\textwidth]{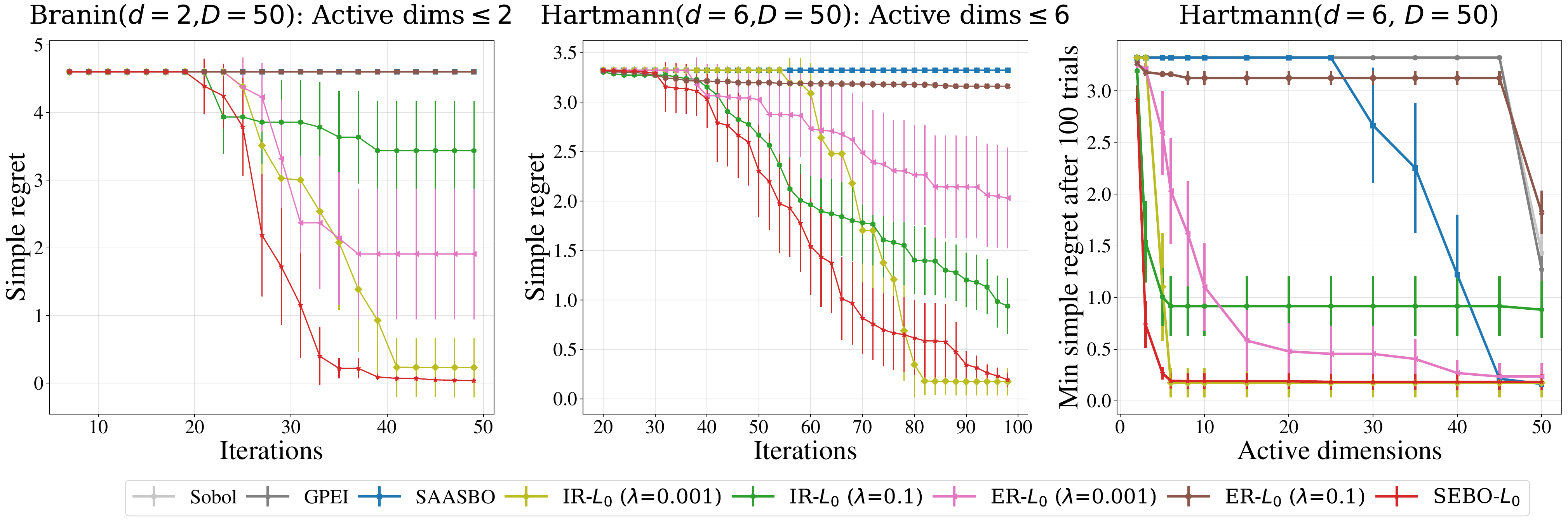}
    \caption{
        (Left) Simple regret for Branin embedded into a $50$D space, considering only observations with at most $2$ active (non-sparse) parameters. 
        SEBO-$L_0$ performed the best followed by IR with $\lambda=0.001$. 
        (Middle) SEBO-$L_0$ and IR with $\lambda=0.001$ performed the best for the Hartmann6 function embedded into a $50$D space when considering only observations with at most $6$ active parameters.
        (Right) The objective-sparsity trade-off after all $100$ iterations on the Hartmann6 problem. 
        Shown is the \emph{Pareto frontier} between sparsity and simple regret after the evaluation budget has been exhausted.
        SEBO-$L_0$ is able to explore the trade-offs and is able to discover sparse configurations with fewer than $6$ active parameters that are not found by the other methods.
    }
    \label{fig:synthetic}
\end{figure*}
\vspace{-1em}

\paragraph{Synthetic functions:}\label{sec:synthetic}
We first consider two synthetic problems where the level of sparsity is known. 
We use the Branin and Hartmann6 functions embedded into a $50$D space where $0$ is considered sparse, i.e. $\boldx^s$ = $\bold0$. We used $50$ trials (evaluations) with $8$ quasi-random initial points for Branin and $100$ trials with $20$ quasi-random initial points for Hartmann6.
The results are shown in Fig.~\ref{fig:synthetic}. 
The two leftmost plots show the optimization results by evaluating the objective only on observed points whose number of active (i.e., non-zero) parameters was less than or equal to the true effective dimension ($2$ for Branin and $6$ for Hartmann6). 

We observe that SEBO-$L_0$ performed the best, followed by IR with $\lambda=0.001$. 
This suggests IR may perform competitively if the regularization coefficient is chosen optimally. 
On the other hand, ER performed worse than SEBO and IR. Finally, methods with non-regularized acquisition functions (Sobol, GPEI, and SAASBO) failed to identify sparse configurations since they do not explicitly optimize for sparsity of the solutions. 
Fig.~\ref{fig:synthetic} (Right) visualizes the trade-off between the objective and sparsity and SEBO-$L_0$ yielded the best sparsity trade-offs. 

\paragraph{Ranking sourcing system simulation:}\label{sec:sourcing}
The sourcing component of a recommendation system is responsible for retrieving a collection of items that are sent to the ranking algorithm for scoring. Items are retrieved from multiple sources, for instance that may represent different aspects of the user interest taxonomy~\citep{wilhelm2018dpp}.
Querying for more items can potentially improve the quality of the recommendation system, but comes at the cost of increasing the infrastructure load. 
In addition, each source may require individual maintenance; thus, deprecating poor sources could reduce technical debt and maintenance costs of an entire recommendation system~\citep{techdebt}. 
Our goal is thus to identify a retrieval policy that uses a minimal number of sources while still maximizing the ranking quality score, measured by a function of content relevance and infrastructure load.

We developed a simulation of a recommender sourcing system that simulates the quality and infrastructure load of recommendations produced by a particular sourcing policy. 
The sourcing system is modeled as a topic model, where each source has a different distribution over topics, and topics have different levels of relevance to the user. 
When two sources are (topically) similar, they may obtain duplicate items, which will not improve recommendation quality. 

We consider a $25$D retrieval policy in which each parameter specifies the number of items retrieved from a particular source. 
Our desired sparsity is to set parameters to $0$ ($\boldx^s$ = $\bold0$), i.e., turning off the source. 
See Sec.~\ref{sec:sourcing_setup} for more details.  
We used $8$ initial points and ran $100$ trials for all the methods. 
Fig.~\ref{fig:real_world} (Left) shows that SEBO-$L_0$ performed the best in optimizing the ranking quality score under different sparsity levels. 
Sobol and GPEI could not find sparse policies and obtained worse quality scores even with $25$ active parameters. 
IR and SAASBO performed similarly, and ER with the larger regularization parameter $\lambda=0.01$ achieved higher quality score with less than $10$ active dimensions. 

\begin{figure*}[!ht]
    \centering
    \includegraphics[width=0.95\textwidth]{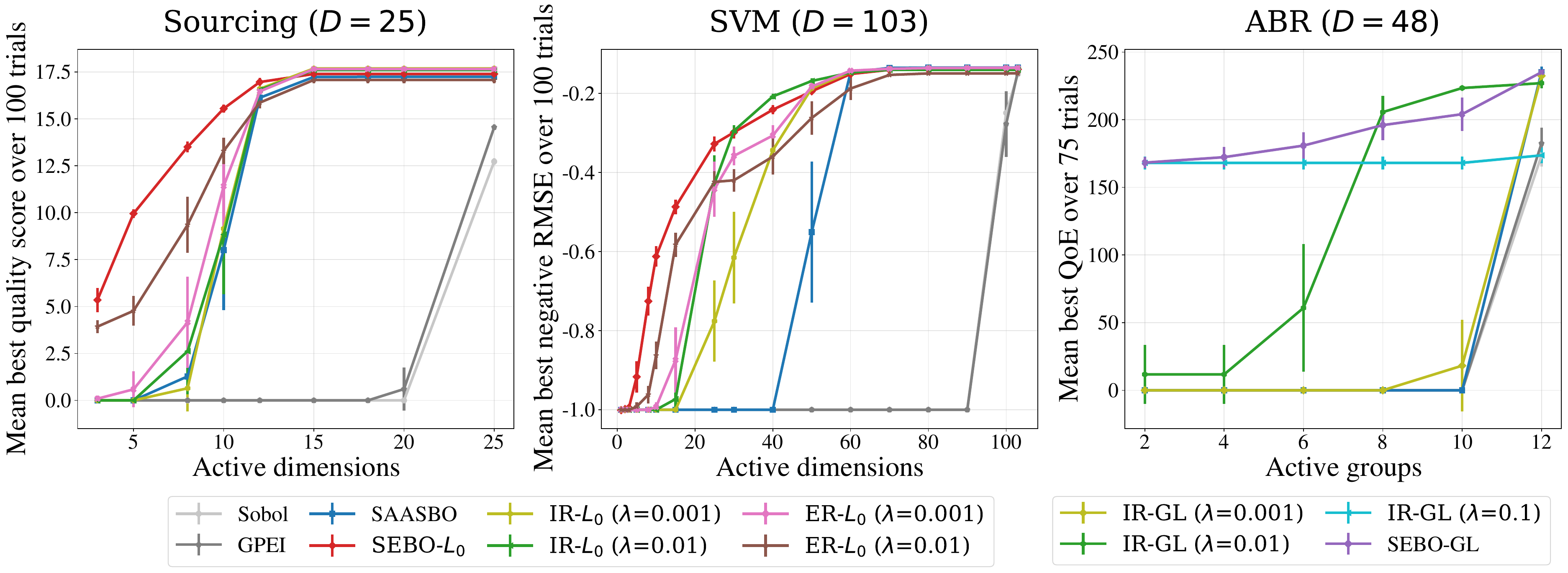}
    \caption{
        Objective-sparsity trade-offs after $100$ ($75$ for ABR) trials for the three real-world problems.
        (Left) \textit{Sourcing problem}: SEBO-$L_0$ regularization effectively explored all sparsity trade-offs. 
        (Middle) \textit{SVM problem}: ER with $\lambda=0.01$ and IR with $\lambda=0.01$ were able to explore parts of the Pareto frontier, however were dominated by SEBO-$L_0$. 
        (Right) \textit{ABR problem}: Similar behavior as in the SVM problem was seen here with a group lasso penalty.
    }
    \label{fig:real_world}
\end{figure*}
\vspace{-1em}

\paragraph{SVM Machine learning hyperparameter tuning:}
We consider the problem of doing joint feature selection and hyperparameter tuning for a support vector machine (SVM). 
We tuned the $C$, $\epsilon$, and $\gamma$ hyperparameters of the SVM, jointly with separate scale factors in the continuous range $[0, 1]$ for each feature.
We used $100$ features from the CT slice UCI dataset~\cite{uci} and the goal was to minimize the RMSE on the test set.
This produces a $103$D optimization problem where we shrink each feature towards a scale factor of $0$, i.e. $\boldx^s_i = 0$, as it effectively removes the feature from the dataset.
We took $C \in [0.01, 1.0]$, $\epsilon \in [0.01, 1.0]$, and $\gamma \in [0.001, 0.1]$, where the center of each interval was considered sparse as this is the default value in Sklearn {(i.e. $\boldx^s_i = \text{Mid}(\text{Hyperparameter Interval})$)}.
We optimized $C$, $\epsilon$, and $\gamma$ on a log-scale, and initialized all methods with $20$ points and ran $100$ evaluations. Fig.~\ref{fig:real_world} (Middle) shows that SEBO-$L_0$ was best able to explore the trade-offs between sparsity and (negative) RMSE.

\begin{figure*}[ht]
        \centering
        \includegraphics[width=0.925\textwidth]{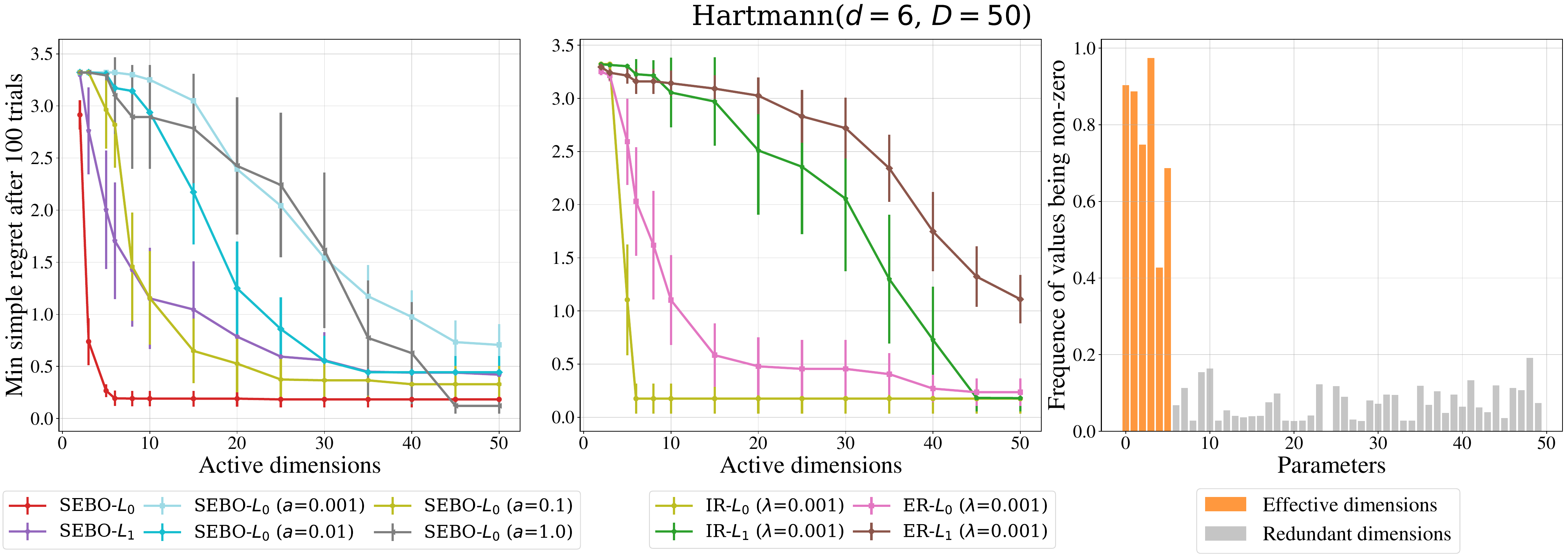}
    \caption{
        Ablation study on the Hartmann6 function embedded in a $50$D space. 
        (Left) SEBO-$L_0$ works much better than SEBO-$L_1$ as it directly targets $L_0$ sparsity. 
        Using a fixed value of $a$ performs poorly, confirming the importance of our homotopy continuation approach.
        (Middle) Working directly with $L_0$ regularization works drastically better than $L_1$ regularization for both IR and ER.
        (Right) The $6$ important parameters are more frequently included in Pareto optimal configurations for the embedded Hartmann6 problem.
    }
    \label{fig:ablation}
\end{figure*}
\vspace{-1em}

\paragraph{Adaptive bitrate simulation:}\label{sec:abr}
Video streaming and real-time conferencing systems use adaptive bitrate (ABR) algorithms to balance video quality and uninterrupted playback. 
The goal is to maximize the quality of experience (QoE). 
The optimal policy for a particular ABR controller may depend on the network, for instance a stream with large fluctuations in bandwidth will benefit from different ABR parameters than a stream with stable bandwidth.
This motivates the use of a contextual policy where ABR parameters are personalized by context variables such as country or network type~\citep{feng2020cbo}.
Various other systems and infrastructure applications commonly rely on tunable parameters which can benefit from contextualization. 

We suppose that the system has already been optimized with a global non-contextual policy, ${\pi}_\text{global}$, that is used for all contexts. Our goal here is to use sparse BO to find the contextualized residuals ${\Delta\pi}_i$ for each individual context $i$, i.e., $\pi_i = \pi_\text{global} + {\Delta\pi}_i$, with the target sparse point $\boldx^s$ set to be ${\pi}_\text{global}$.
By regularizing the contextualized residuals ${\Delta\pi}_i$'s using the group lasso (GL) norm~\citep{grouplasso}, we hope to find policies that require minimum alteration to the global policy ${\pi}_\text{global}$, in which the minimum number of contexts have parameters that deviate from the global optimum. 
This adds both simplicity and interpretability to the contextual policy, since we can interpret the policy by looking at the contextual residuals ${\Delta\pi}_i$. 

Fig.~\ref{fig:real_world} (Right) shows the results of applying our methods to the contextual ABR optimization problem from~\citet{feng2020cbo}. 
For this problem, we have $12$ contexts and $4$ parameters for each context resulting in a $48$D optimization problem. 
We used $75$ trials with $8$ quasi-random initial points for all the methods. 
The group lasso penalty is defined by assigning parameters for each individual context to be within the same group. 
We observe that IR with a fixed $\lambda$ was able to explore trade-offs at certain sparsity levels and that stronger regularization (larger $\lambda$) resulted in finding configurations that were more sparse.
SEBO-GL, on the other hand, automatically and efficiently explored the trade-off between sparsity and reward at all sparsity levels. 
All other baselines (Sobol, GPEI, SAASBO) failed to find any sparse configurations that achieve non-zero reward.

\paragraph{Ablation study and Interpretation:}
We show by means of an ablation study the importance of using the homotopy continuation approach from Section~\ref{sec:l0_sparsity} to target $L_0$ sparsity.
We focus on SEBO as it consistently outperformed IR and ER, and refer to Fig.~\ref{fig:compare_gp} in Appendix \ref{sec: ablation saas} for additional results on the importance of using the SAAS model.
The results from the ablation study can be seen in Fig.~\ref{fig:ablation}.
Using a fixed value of $a$ for the $L_0$ approximation performs poorly, particularly when $a$ is small, which is due to the acquisition function being zero almost everywhere and thus difficult to optimize.
On the other hand, $a=1$ results in a failure to discover sparse configurations and the resulting method performs similar to SAASBO (see Fig.~\ref{fig:synthetic}).
In addition, we show that for all approaches (ER, IR, and SEBO), working directly with $L_0$ regularization works significantly better than the frequently used $L_1$ regularization. 
Finally, we show in Fig.~\ref{fig:ablation} (Right) how frequently each parameter is turned on (non-zero) in the final Pareto frontier for each replication of SEBO-$L_0$, which indicates the method correctly identifies the important parameters. See Appendix~\ref{sec: sparse_sol} for more interpretations of SEBO-$L_0$ configurations in other benchmarks.

\section{DISCUSSION}
\label{sec:discussion}
BO is a powerful tool for sample-efficient optimization of real-world systems. 
Recent developments in BO have made it possible to optimize hundreds of parameters, providing solutions to complex optimization problems in science and engineering.  
Yet practitioners and decision-makers often favor simplicity in the solutions, e.g., in the design space, for the sake of interpretability, managing risk, or for reducing technical debt.
This poses a new challenge: how should we discover well-performing \textit{and} parsimonious designs in a sample-efficient manner? 

We show that sparsity-inducing models are not sufficient for producing sparse designs, and examine several schemes for penalizing design parameters within the acquisition function itself.  
We utilize theoretical insights from multi-objective optimization to identify limitations of common penalization approaches and propose SEBO, which optimizes for both sparsity and performance. 
In doing so, we are able to learn the entire set of optimal trade-offs between objective and sparsity, allowing decision makers to select the amount of objective they are willing to sacrifice for increased interpretability and simplicity.

Our formulation is compatible with a variety of regularizers, including $L_0$, $L_1$, and the group lasso penalties.
To enable the optimization of the discontinuous $L_0$ penalty, we develop a novel acquisition function optimization method based on homotopy continuation that enables gradient-based optimization.
We find that SEBO with $L_0$ penalization consistently outperforms all other methods in identifying optimal designs, while also eliminating the need to tune regularization hyperparameters.

Our work has a few limitations that suggest areas for future work. 
First, SEBO can be useful for identifying the entire Pareto frontier of sparse solutions, but in some contexts decision-makers may have a desired sparsity level in mind. 
Further work is required to develop adaptive algorithms that can efficiently target specific sparsity levels. 
Second, if the goal is to reduce regret while achieving sparsity, there may be opportunities for theoretical work on selecting model and acquisition function regularization parameters simultaneously, see, e.g.,~\citet{bastani2020lassobandit}.

\bibliographystyle{abbrvnat}
\bibliography{mybib}

\begin{thebibliography}{46}
\providecommand{\natexlab}[1]{#1}
\providecommand{\url}[1]{\texttt{#1}}
\expandafter\ifx\csname urlstyle\endcsname\relax
  \providecommand{\doi}[1]{doi: #1}\else
  \providecommand{\doi}{doi: \begingroup \urlstyle{rm}\Url}\fi

\bibitem[Allgower and Georg(2012)]{allgower2012numerical}
E.~L. Allgower and K.~Georg.
\newblock \emph{Numerical continuation methods: an introduction}, volume~13.
\newblock Springer Science \& Business Media, 2012.

\bibitem[Balandat et~al.(2020)Balandat, Karrer, Jiang, Daulton, Letham, Wilson,
  and Bakshy]{botorch}
M.~Balandat, B.~Karrer, D.~R. Jiang, S.~Daulton, B.~Letham, A.~G. Wilson, and
  E.~Bakshy.
\newblock {BoTorch}: A framework for efficient {M}onte-{C}arlo {B}ayesian
  optimization.
\newblock In \emph{Advances in Neural Information Processing Systems 33},
  NeurIPS, pages 21524--21538, 2020.

\bibitem[Bastani and Bayati(2020)]{bastani2020lassobandit}
H.~Bastani and M.~Bayati.
\newblock Online decision making with high-dimensional covariates.
\newblock \emph{Operations Research}, 68\penalty0 (1):\penalty0 276--294, 2020.

\bibitem[Blei et~al.(2003)Blei, Ng, and Jordan]{blei2003latent}
D.~M. Blei, A.~Y. Ng, and M.~I. Jordan.
\newblock Latent {D}irichlet allocation.
\newblock \emph{Journal of Machine Learning Research}, 3\penalty0
  (Jan):\penalty0 993--1022, 2003.

\bibitem[Bobadilla et~al.(2013)Bobadilla, Ortega, Hernando, and
  Guti{\'e}rrez]{bobadilla2013recommender}
J.~Bobadilla, F.~Ortega, A.~Hernando, and A.~Guti{\'e}rrez.
\newblock Recommender systems survey.
\newblock \emph{Knowledge-Based Systems}, 46:\penalty0 109--132, 2013.

\bibitem[Bowman(1976)]{bowman76}
V.~J. Bowman.
\newblock On the relationship of the {T}chebycheff norm and the efficient
  frontier of multiple-criteria objectives.
\newblock In H.~Thiriez and S.~Zionts, editors, \emph{Multiple Criteria
  Decision Making. Lecture Notes in Economics and Mathematical Systems
  (Operations Research), vol 130}, pages 76--86. Springer Berlin, Heidelberg,
  1976.

\bibitem[Calandra et~al.(2015)Calandra, Seyfarth, Peters, and
  Deisenroth]{Calandra2015a}
R.~Calandra, A.~Seyfarth, J.~Peters, and M.~P. Deisenroth.
\newblock {Bayesian} optimization for learning gaits under uncertainty.
\newblock \emph{Annals of Mathematics and Artificial Intelligence}, 76\penalty0
  (1):\penalty0 5--23, 2015.

\bibitem[Covington et~al.(2016)Covington, Adams, and Sargin]{covington2016deep}
P.~Covington, J.~Adams, and E.~Sargin.
\newblock Deep neural networks for {YouTube} recommendations.
\newblock In \emph{Proceedings of the 10th ACM Conference on Recommender
  Systems}, RecSys, pages 191--198, 2016.

\bibitem[Das and Dennis(1997)]{das97}
I.~Das and J.~Dennis.
\newblock A closer look at drawbacks of minimizing weighted sums of objectives
  for {P}areto set generation in multicriteria optimization problems.
\newblock \emph{Structural Optimization}, 14:\penalty0 63--69, 1997.

\bibitem[Daulton et~al.(2020)Daulton, Balandat, and Bakshy]{daulton2020ehvi}
S.~Daulton, M.~Balandat, and E.~Bakshy.
\newblock Differentiable expected hypervolume improvement for parallel
  multi-objective {Bayesian} optimization.
\newblock In \emph{Advances in Neural Information Processing Systems 33},
  NeurIPS, pages 9851--9864, 2020.

\bibitem[Daulton et~al.(2021)Daulton, Balandat, and Bakshy]{daulton2021nehvi}
S.~Daulton, M.~Balandat, and E.~Bakshy.
\newblock Parallel {Bayesian} optimization of multiple noisy objectives with
  expected hypervolume improvement.
\newblock In \emph{Advances in Neural Information Processing Systems 34},
  NeurIPS, pages 2187--2200, 2021.

\bibitem[Doshi-Velez and Kim(2017)]{doshi2017towards}
F.~Doshi-Velez and B.~Kim.
\newblock Towards a rigorous science of interpretable machine learning.
\newblock \emph{arXiv preprint arXiv:1702.08608}, 2017.

\bibitem[Dua and Graff(2017)]{uci}
D.~Dua and C.~Graff.
\newblock {UCI} machine learning repository, 2017.

\bibitem[Ehrgott(2005)]{Ehrgott05}
M.~Ehrgott.
\newblock \emph{Multicriteria Optimization}.
\newblock Springer Berlin, Heidelberg, 2005.

\bibitem[Eriksson and Jankowiak(2021)]{saasbo}
D.~Eriksson and M.~Jankowiak.
\newblock High-dimensional {B}ayesian optimization with sparse axis-aligned
  subspaces.
\newblock In \emph{Proceedings of the 37th Conference on Uncertainty in
  Artificial Intelligence}, UAI, pages 493--503, 2021.

\bibitem[Eriksson et~al.(2019)Eriksson, Pearce, Gardner, Turner, and
  Poloczek]{eriksson2019turbo}
D.~Eriksson, M.~Pearce, J.~Gardner, R.~D. Turner, and M.~Poloczek.
\newblock Scalable global optimization via local {B}ayesian optimization.
\newblock In \emph{Advances in Neural Information Processing Systems 32},
  NeurIPS, 2019.

\bibitem[Evgeniou et~al.(2002)Evgeniou, Poggio, Pontil, and Verri]{evgeniou02}
T.~Evgeniou, T.~Poggio, M.~Pontil, and A.~Verri.
\newblock Regularization and statistical learning theory for data analysis.
\newblock \emph{Computational Statistics \& Data}, 38\penalty0 (4):\penalty0
  421--432, 2002.

\bibitem[Felton et~al.(2021)Felton, Rittig, and Lapkin]{Felton2021}
K.~Felton, J.~Rittig, and A.~Lapkin.
\newblock Summit: Benchmarking machine learning methods for reaction
  optimisation.
\newblock \emph{Chemistry Methods}, 1\penalty0 (2):\penalty0 116--122, 2021.

\bibitem[Feng et~al.(2020)Feng, Letham, Mao, and Bakshy]{feng2020cbo}
Q.~Feng, B.~Letham, H.~Mao, and E.~Bakshy.
\newblock High-dimensional contextual policy search with unknown context
  rewards using {B}ayesian optimization.
\newblock In \emph{Advances in Neural Information Processing Systems 33},
  NeurIPS, pages 22032--22044, 2020.

\bibitem[G\'{o}mez-Bombarelli et~al.(2018)G\'{o}mez-Bombarelli, Wei, Duvenaud,
  Hern\'{a}ndez-Lobato, S\'{a}nchez-Lengeling, Sheberla, Aguilera-Iparraguirre,
  Hirzel, Adams, and Aspuru-Guzik]{gomez2018chemical}
R.~G\'{o}mez-Bombarelli, J.~N. Wei, D.~Duvenaud, J.~M. Hern\'{a}ndez-Lobato,
  B.~S\'{a}nchez-Lengeling, D.~Sheberla, J.~Aguilera-Iparraguirre, T.~D.
  Hirzel, R.~P. Adams, and A.~Aspuru-Guzik.
\newblock Automatic chemical design using a data-driven continuous
  representation of molecules.
\newblock \emph{ACS Central Science}, 4\penalty0 (2):\penalty0 268--276, 2018.

\bibitem[Gonz{\'a}lez et~al.(2016)Gonz{\'a}lez, Dai, Hennig, and
  Lawrence]{gonzalez2016batch}
J.~Gonz{\'a}lez, Z.~Dai, P.~Hennig, and N.~Lawrence.
\newblock Batch {B}ayesian optimization via local penalization.
\newblock In \emph{Proceedings of the 19th International Conference on
  Artificial Intelligence and Statistics}, AISTATS, pages 648--657, 2016.

\bibitem[Hern\'{a}ndez-Lobato et~al.(2014)Hern\'{a}ndez-Lobato, Hoffman, and
  Ghahramani]{hernandezlobato2014pes}
J.~M. Hern\'{a}ndez-Lobato, M.~W. Hoffman, and Z.~Ghahramani.
\newblock Predictive entropy search for efficient global optimization of
  black-box functions.
\newblock In \emph{Advances in Neural Information Processing Systems 27}, NIPS,
  pages 918--926, 2014.

\bibitem[Hu et~al.(2019)Hu, Rudin, and Seltzer]{HuRuSe2019}
X.~Hu, C.~Rudin, and M.~Seltzer.
\newblock Optimal sparse decision trees.
\newblock In \emph{Advances in Neural Information Processing Systems 32},
  NeurIPS, pages 7267--7275, 2019.

\bibitem[Hutter et~al.(2011)Hutter, Hoos, and Leyton-Brown]{hutter2011smac}
F.~Hutter, H.~H. Hoos, and K.~Leyton-Brown.
\newblock Sequential model-based optimization for general algorithm
  configuration.
\newblock In \emph{International Conference on Learning and Intelligent
  Optimization}, LION, pages 507--523, 2011.

\bibitem[Jones et~al.(1998)Jones, Schonlau, and Welch]{jones98}
D.~R. Jones, M.~Schonlau, and W.~J. Welch.
\newblock Efficient global optimization of expensive black-box functions.
\newblock \emph{Journal of Global Optimization}, 13:\penalty0 455--492, 1998.

\bibitem[Knowles(2006)]{parego}
J.~Knowles.
\newblock {ParEGO}: A hybrid algorithm with on-line landscape approximation for
  expensive multiobjective optimization problems.
\newblock \emph{IEEE Transactions on Evolutionary Computation}, 10\penalty0
  (1):\penalty0 50--66, 2006.

\bibitem[Letham et~al.(2019)Letham, Karrer, Ottoni, and
  Bakshy]{letham2019noisyei}
B.~Letham, B.~Karrer, G.~Ottoni, and E.~Bakshy.
\newblock Constrained {B}ayesian optimization with noisy experiments.
\newblock \emph{{Bayesian} Analysis}, 14\penalty0 (2):\penalty0 495--519, 2019.

\bibitem[Lizotte et~al.(2007)Lizotte, Wang, Bowling, and
  Schuurmans]{Lizotte2007Automatic}
D.~J. Lizotte, T.~Wang, M.~Bowling, and D.~Schuurmans.
\newblock Automatic gait optimization with {G}aussian process regression.
\newblock In \emph{Proceedings of the 20th International Joint Conference on
  Artificial Intelligence}, IJCAI, pages 944--949, 2007.

\bibitem[Marler and Arora(2010)]{marler10}
R.~T. Marler and J.~S. Arora.
\newblock The weighted sum method for multi-objective optimization: new
  insights.
\newblock \emph{Structural and Multidisciplinary Optimization}, 41:\penalty0
  853--862, 2010.

\bibitem[Oh et~al.(2019)Oh, Tomczak, Gavves, and Welling]{oh19}
C.~Oh, J.~M. Tomczak, E.~Gavves, and M.~Welling.
\newblock Combinatorial {B}ayesian optimization using the graph {C}artesian
  product.
\newblock In \emph{Advances in Neural Information Processing Systems 32},
  NeurIPS, pages 2914--2924, 2019.

\bibitem[Park et~al.(2021)Park, Borth, Wilson, and Hunter]{park21}
C.~Park, D.~J. Borth, N.~S. Wilson, and C.~N. Hunter.
\newblock Variable selection for {G}aussian process regression through a sparse
  projection.
\newblock \emph{IISE Transactions}, 54\penalty0 (7):\penalty0 699--712, 2021.

\bibitem[Rudin et~al.(2022)Rudin, Chen, Chen, Huang, Semenova, and
  Zhong]{rudin22}
C.~Rudin, C.~Chen, Z.~Chen, H.~Huang, L.~Semenova, and C.~Zhong.
\newblock Interpretable machine learning: fundamental principles and 10 grand
  challenges.
\newblock \emph{Statistics Surveys}, 16:\penalty0 1--85, 2022.

\bibitem[Sculley et~al.(2015)Sculley, Holt, Golovin, Davydov, Phillips, Ebner,
  Chaudhary, Young, Crespo, and Dennison]{techdebt}
D.~Sculley, G.~Holt, D.~Golovin, E.~Davydov, T.~Phillips, D.~Ebner,
  V.~Chaudhary, M.~Young, J.-F. Crespo, and D.~Dennison.
\newblock Hidden technical debt in machine learning systems.
\newblock In \emph{Advances in Neural Information Processing Systems 28}, NIPS,
  pages 2503--2511, 2015.

\bibitem[Shahriari et~al.(2015)Shahriari, Swersky, Wang, Adams, and
  de~Freitas]{shahriari2015review}
B.~Shahriari, K.~Swersky, Z.~Wang, R.~P. Adams, and N.~de~Freitas.
\newblock Taking the human out of the loop: a review of {B}ayesian
  optimization.
\newblock \emph{Proceedings of the IEEE}, 104\penalty0 (1):\penalty0 148--175,
  2015.

\bibitem[Shahriari et~al.(2016)Shahriari, Bouchard-C{\^o}t{\'e}, and
  de~Freitas]{shahriari2016unbounded}
B.~Shahriari, A.~Bouchard-C{\^o}t{\'e}, and N.~de~Freitas.
\newblock Unbounded {B}ayesian optimization via regularization.
\newblock In \emph{Proceedings of the 19th International Conference on
  Artificial Intelligence and Statistics}, AISTATS, pages 1168--1176, 2016.

\bibitem[Snoek et~al.(2012)Snoek, Larochelle, and Adams]{snoek2012practical}
J.~Snoek, H.~Larochelle, and R.~P. Adams.
\newblock Practical {B}ayesian optimization of machine learning algorithms.
\newblock In \emph{Advances in Neural Information Processing Systems 25}, NIPS,
  pages 2951--2959, 2012.

\bibitem[Srinivas et~al.(2010)Srinivas, Krause, Kakade, and Seeger]{srinivas10}
N.~Srinivas, A.~Krause, S.~Kakade, and M.~Seeger.
\newblock {Gaussian} process optimization in the bandit setting: no regret and
  experimental design.
\newblock In \emph{Proceedings of the 27th International Conference on Machine
  Learning}, ICML, pages 1015--1022, 2010.

\bibitem[Tibshirani(1996)]{lasso}
R.~Tibshirani.
\newblock Regression shrinkage and selection via the lasso.
\newblock \emph{Journal of the Royal Statistical Society: Series B},
  58\penalty0 (1):\penalty0 267--288, 1996.

\bibitem[Turner et~al.(2021)Turner, Eriksson, McCourt, Kiili, Laaksonen, Xu,
  and Guyon]{turner2021bayesian}
R.~Turner, D.~Eriksson, M.~McCourt, J.~Kiili, E.~Laaksonen, Z.~Xu, and
  I.~Guyon.
\newblock {Bayesian} optimization is superior to random search for machine
  learning hyperparameter tuning: analysis of the black-box optimization
  challenge 2020.
\newblock In \emph{NeurIPS 2020 Competition and Demonstration Track}, pages
  3--26, 2021.

\bibitem[Ustun and Rudin(2016)]{UstunRu2016SLIM}
B.~Ustun and C.~Rudin.
\newblock Supersparse linear integer models for optimized medical scoring
  systems.
\newblock \emph{Machine Learning}, 102\penalty0 (3):\penalty0 349--391, 2016.

\bibitem[{Wang} and {Jegelka}(2017)]{wang2017maxvalue}
Z.~{Wang} and S.~{Jegelka}.
\newblock Max-value entropy search for efficient {B}ayesian optimization.
\newblock In \emph{Proceedings of the 34th International Conference on Machine
  Learning}, ICML, pages 3627--3635, 2017.

\bibitem[Wang et~al.(2016)Wang, Hutter, Zoghi, Matheson, and
  de~Freitas]{wang2013bayesian}
Z.~Wang, F.~Hutter, M.~Zoghi, D.~Matheson, and N.~de~Freitas.
\newblock Bayesian optimization in a billion dimensions via random embeddings.
\newblock \emph{Journal of Artificial Intelligence Research}, 55:\penalty0
  361--387, 2016.

\bibitem[Wilhelm et~al.(2018)Wilhelm, Ramanathan, Bonomo, Jain, Chi, and
  Gillenwater]{wilhelm2018dpp}
M.~Wilhelm, A.~Ramanathan, A.~Bonomo, S.~Jain, E.~H. Chi, and J.~Gillenwater.
\newblock Practical diversified recommendations on {YouTube} with determinantal
  point processes.
\newblock In \emph{Proceedings of the 27th ACM International Conference on
  Information and Knowledge Management}, CIKM, pages 2165--2173, 2018.

\bibitem[Yang et~al.(2019)Yang, Emmerich, Deutz, and B\"{a}ck]{yang2019}
K.~Yang, M.~Emmerich, A.~Deutz, and T.~B\"{a}ck.
\newblock Multi-objective {B}ayesian global optimization using expected
  hypervolume improvement gradient.
\newblock \emph{Swarm and Evolutionary Computation}, 44:\penalty0 945--956,
  2019.

\bibitem[Yuan and Lin(2006)]{grouplasso}
M.~Yuan and Y.~Lin.
\newblock Model selection and estimation in regression with grouped variables.
\newblock \emph{Journal of the Royal Statistical Society: Series B},
  68\penalty0 (1):\penalty0 49--67, 2006.

\bibitem[Zhang(2008)]{zhang2008multi}
T.~Zhang.
\newblock Multi-stage convex relaxation for learning with sparse
  regularization.
\newblock In \emph{Advances in Neural Information Processing Systems 21}, NIPS,
  2008.

\end{thebibliography}

\clearpage
\onecolumn

\hsize\textwidth\linewidth\hsize\toptitlebar
{\centering{\Large\bfseries Sparse Bayesian optimization: Supplementary Material \par}}
\bottomtitlebar

\appendix

\section{POTENTIAL SOCIETAL IMPACTS}
\label{appdx:society}

BO is often used to optimize complicated black-box functions such as training deep neural networks, tuning recommendation systems, designing molecules, or synthesizing compounds in chemistry.
Our method enables finding sparse solutions while optimizing the objective of interest.
In many situations, a sparse solution can help reduce tech debt as well as making it easier to interpret.
Our SEBO method is able to automatically explore the trade-offs between the objective(s) and sparsity which will allow the decision-maker to choose a solution of their liking.
Lastly, as the black-box functions are often expensive to evaluate, the sample-efficiency of our method may reduce the environmental impact compared to using a less sample-efficient method.

\section{CODE IMPLEMENTATIONS}
\label{appdx:licenses}

The GPEI, SAASBO and EHVI used in SEBO were implemented using BoTorch, a framework for BO in PyTorch~\cite{botorch} and are available in Ax \url{https://github.com/facebook/Ax}. The code is licensed under the MIT License.
The SVM hyperparameter tuning experiment uses the SVM implementation in Sklearn and the CT slice dataset in the UCI machine learning repository~\cite{uci}.
The Adaptive bitrate simulation experiment is available at \url{https://github.com/facebookresearch/ContextualBO}, licensed under the MIT License.

\section{THEORETICAL RESULTS}
\label{sec:proofs}
Here we provide the proofs of Propositions \ref{thm:ext} and \ref{thm:int}, as well as an illustration of the result of Proposition \ref{thm:int}.

\begin{proof}[Proof of Proposition \ref{thm:ext}]
Suppose $\mathbf{x}^{\dagger} \in \arg \max \alpha_\textrm{ER}(\mathbf{x} ; \lambda)$ and $\xi(\mathbf{x}^{\dagger}) \leq \theta$. Then, $\alpha(\mathbf{x}^{\dagger}) = 0$, so $\alpha_\textrm{ER}(\mathbf{x}^{\dagger}; \lambda) = -\lambda \xi(\mathbf{x}^{\dagger})$.

By $\mathbf{x}^{\dagger}$ being a maximizer of $\alpha_\textrm{ER}$ we must have
\begin{equation*}
    -\lambda \xi(\mathbf{x}^{\dagger}) = \alpha_\textrm{ER}(\mathbf{x}^{\dagger} ; \lambda) \geq \alpha_\textrm{ER}(\mathbf{x}^s ; \lambda) = -\lambda \xi(\mathbf{x}^s).
\end{equation*}
Thus $\xi(\mathbf{x}^{\dagger}) \leq \xi(\mathbf{x}^s)$. Because $\mathbf{x}^s$ is a strict global minimum, we have then that $\mathbf{x}^{\dagger} = \mathbf{x}^s$.
\end{proof}
This setting where the acquisition value is $0$ for all sparse points is easily encountered in practice when there is a trade-off between the objective function and sparsity, as in Fig. \ref{fig:sourcing_intro}, and we have sampled a point close to the (non-sparse) optimum. Consider the EI acquisition function with external regularization. Once the GP is confident that sparse points have worse objective value than non-sparse points, sparse points will have acquisition value approximately $0$, as their improvement is being evaluated with respect to a non-sparse incumbent best $\mathbf{x}^*$.

We assume $\xi$ is continuous and bounded, which implies $h$ is continuous and bounded:
\begin{assumption}\label{ass:cont}
    $\xi$ is continuous on $\mathcal{B}$, and has minimum value $\xi(\mathbf{x}^s) = s_l$ and maximum value $s_u$.
\end{assumption}
\begin{proposition}
    $h$ is continuous and bounded on the domain $[s_l, s_u]$.
\end{proposition}
\begin{proof}[Sketch of Proof]
    This result falls from the continuity and boundedness of $f$, and by applying the intermediate value theorem to $\xi$.
\end{proof}

\begin{proof}[Proof of Proposition \ref{thm:int}]
Suppose $h$ is strictly convex over the interval $[\theta_l, \theta_u]$. For the sake of contradiction, assume that there exists a $\theta_\dagger \in (\theta_l, \theta_u)$ and an $\mathbf{x}^\dagger$ such that $\mathbf{x}^\dagger \in \arg \max g(\mathbf{x}; \lambda)$ and $\xi(\mathbf{x}^\dagger) = \theta_\dagger$.

It is clear that $\mathbf{x}^\dagger \in \arg \max f(\mathbf{x}) \textrm{ subject to } \xi(\mathbf{x}) = \theta_\dagger$, otherwise the point with strictly larger $f$ and equal $\xi$ value would have a higher value for $g$, and  $\mathbf{x}^\dagger$ could not be optimal for $g$. Thus, $f(\mathbf{x}^\dagger) = h(\theta_\dagger)$.

We can express $\theta_\dagger = t \theta_l + (1 -t) \theta_u$ for some $t \in (0, 1)$. By strict convexity of $h$ on this interval, we have that
\begin{equation}\label{eq:convex}
    h(\theta_\dagger) < t h(\theta_l) + (1-t) h(\theta_u).
\end{equation}
Take $\mathbf{x}^u \in \arg \max f(\mathbf{x}) \textrm{ subject to } \xi(\mathbf{x}) = \theta_u$, and $\mathbf{x}^l \in \arg \max f(\mathbf{x}) \textrm{ subject to } \xi(\mathbf{x}) = \theta_l$. These are the points in $\mathcal{B}$ corresponding to $h(\theta_l)$ and $h(\theta_u)$. The optimality of $\mathbf{x}^\dagger$ implies that $g(\mathbf{x}^\dagger ; \lambda) \geq g(\mathbf{x}^u ; \lambda)$ and $g(\mathbf{x}^\dagger ; \lambda) \geq g(\mathbf{x}^l ; \lambda)$. Thus,
\begin{align}\nonumber
    g(\mathbf{x}^\dagger ; \lambda) &\geq t g(\mathbf{x}^l ; \lambda) + (1-t) g(\mathbf{x}^u ; \lambda)\\\nonumber
    h(\theta_\dagger) - \lambda \theta_\dagger &\geq t h(\theta_l) - t \lambda \theta_l + (1-t) h(\theta_u) - (1-t) \lambda \theta_u\\\label{eq:endproof}
    h(\theta_\dagger) &\geq t h(\theta_l) + (1-t) h(\theta_u),
\end{align}
using $\theta_\dagger = t \theta_l + (1 -t) \theta_u$. The result in (\ref{eq:endproof}) contradicts the convexity in (\ref{eq:convex}), and so $\mathbf{x}^\dagger$ cannot be optimal for $g$.
\end{proof}

Fig.~\ref{fig:ir_res_plot} shows an illustration of the result of Proposition \ref{thm:int} on a log-transformed version of the classic Branin problem, where $f(x_1, x_2) = -\log(10 + \textrm{Branin}(x_1, x_2))$, and we are using a traditional $L_1$ regularization penalty, $\xi(x_1, x_2) = |x_1| + |x_2|$.
The right panel shows $h(\theta)$, from (\ref{eq:h}), as it traces the trade-off from the minimum of $\xi$ to the maximum of $f$.
There is a wide interval of $L_1$-norm values in the middle, $0.4$ to $2.7$, where $h(\theta)$ is strictly convex.
By Proposition \ref{thm:int}, there is no value of $\lambda$ under which the maximizer of (\ref{eq:linscal}) has $L_1$ norm in that range.
That range of sparsity levels thus cannot be reached by maximizing the regularized function $g$.

\begin{figure}
\centering
    \includegraphics{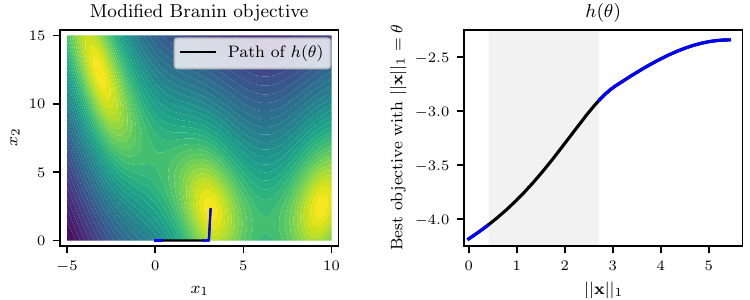}
    \caption{
        An illustration of the internal regularization result in Proposition \ref{thm:int}.
        (Left) The objective $f$ is a modified Branin function.
        The sparsity penalty $\xi$ is the $L_1$ norm.
        (Right) The optimal objective vs. sparsity trade-off, $h(\theta)$, shows the best-achievable objective value for any specified value of $L_1$ norm.
        The shaded region is an interval where $h$ is strictly convex.
        By Proposition \ref{thm:int}, the regularized function in (\ref{eq:linscal}) has no maximizers with $L_1$ norm in that range, for any value of $\lambda$.
    }
    \label{fig:ir_res_plot}
\end{figure}

\section{RELATIONSHIP BETWEEN PAREGO AND INTERNAL REGULARIZATION}
\label{sec:SuppParEGO}
As described in Section~\ref{sec:background}, ParEGO applies the EI acquisition function to a random scalarization of multiple objectives. With internal regularization, random sampling of $\lambda$ for each acquisition optimization produces a ParEGO-style strategy for sparse BO, that differs only in the form of the scalarization.

The inability of linear scalarizations to capture the entire Pareto front, seen in Proposition \ref{thm:int}, is a well-known failure mode for MOO. This result has inspired a large number of alternative scalarizations~\citep{das97}. ParEGO avoids the issue by replacing the linear scalarization with an augmented Chebyshev scalarization \citep{bowman76}. When applied to the context of sparse regularization, this means maximizing
\begin{equation*}\label{eq:cheby}
    T(\mathbf{x}; \lambda) = C (f(\mathbf{x}) - \lambda \xi(\mathbf{x})) - \max(f^* - f(\mathbf{x}), \lambda (\xi(\mathbf{x}) - \xi(\mathbf{x}^s)),
\end{equation*}
where $f^*$ is an estimate for the maximum of $f$ and $C$ is a constant, usually set to $0.05$.
Unlike $g$ in (\ref{eq:linscal}), maximizers of $T$ span the entire objective vs. sparsity trade-off~\citep{parego}.
Using EI to optimize this regularized function with randomly sampled values of $\lambda$ is equivalent to applying ParEGO to the objective and the (negative) sparsity penalty.

\section{OPTIMIZATION WITH \texorpdfstring{$L_0$}{L0} SPARSITY}
\subsection{Homotopy continuation}
In this section we provide some additional details for the homotopy continuation described in Sec.~\ref{sec:l0_sparsity}.
For computational reasons, we use a sequence of $30$ $a$’s starting from $a_{\text{start}} = 10^{-0.5}$ and ending at $10^{-3}$ that is linearly spaced on a log-scale.
First, we optimize the acquisition function using L-BFGS-B from $20$ different starting points to obtain $20$ local optima of $H(x, a_{\text{start}})$.
We then increment the value of $a$ and use L-BFGS-B to re-optimize the homotopy starting from each of the previously found $20$ local optima.
This process is continued until we reach $a=0$ which is the acquisition function corresponding to the true $L_0$ norm.
Note that this procedure traces $20$ curves $c(a) \in \argmin_x H(x, a)$ from $a=a_{\text{start}}$ to $a=0$ and that this curve is of finite length under the assumption that the domain is compact.
These curves are potentially different as the acquisition function may be non-convex and have multiple local optima.
Finally, we choose the candidate as the point that achieves the best acquisition function value.

We use $a_{\text{start}} = 10^{-0.5}$ as it strikes a balance between being large enough to find initial points with non-zero acquisition function values, and being small enough to discover points that are almost sparse.
To better understand this choice note that $\max_{x, z \in [0, 1]} \left|\varphi'_{10^{-0.5}}(\mathbf{x - z})\right| \approx 0.067$ while, e.g., $\max_{x, z \in [0, 1]} \left|\varphi'_{0.1}(\mathbf{x - z})\right| \approx 2 \times 10^{-20}$ which shows that $0.1$ may be too small to serve as $a_{\text{start}}$.
We also investigate this choice in an ablation study in Appendix~\ref{apdx:a_ablation} and find that the performance of SEBO-$L_0$ is not sensitive to the choice of $a_{\text{start}}$ as long as the value is not too small.

\subsection{SEBO algorithm}
\label{subsec:supp-sebo}
The SEBO-$L_0$ method is described in Algorithm~\ref{algo: sebo}.
We start with an initial space-filling experiment design.
In each iteration step, we fit a SAAS GP model and optimize the acquisition function to find the next point to evaluate, as shown at line~\ref{algo:sebo-main}.
When optimizing the acquisition function, homotopy continuation is used to handle the discontinuous $L_0$ norm.
This part is shown on line~\ref{algo:homotopy}.
\begin{algorithm}
\caption{Sparsity Exploring Bayesian Optimization with $L_{0}$ norm (SEBO-$L_{0}$) }
\label{algo: sebo}
\begin{algorithmic}[1]
\Procedure{SEBO-$L_0$}{} \Comment{Outer loop of BO}
\label{algo:sebo-main}
\State Place a Gaussian Process prior on $f$
\State Observe $f$ at $n_0$ quasi-random initial points and get the initial dataset $\mathcal{D}_{n_0}$
\For{$n \gets {n_0 + 1}$ to $N$}
\State Update the posterior probability distribution on $f$ using observed data $\mathcal{D}_{n-1}$
\State Select the next point $\mathbf{x}_{n} \gets$ \Call{Optimize-Homotopy}{$\hat{f_n}$}
\State Evaluate $\mathbf{x}_{n}$: $\mathcal{D}_{n} \gets \{\mathcal{D}_{n-1}, (\mathbf{x}_{n}, f(\mathbf{x}_{n}))\}$
\EndFor
\State \textbf{return} The best point
\EndProcedure
\vspace{1em}
\Procedure{Optimize-Homotopy}{$\hat{f}$}  \Comment{Optimize SEBO-$L_0$ acquisition function}
\label{algo:homotopy}
    \State Define a homotopy $H(\mathbf{x}, a)$ using the posterior on $f$
    \State Initialize a candidate pool $\mathcal{X}_{a} \gets \{\}$
    \For{$a \gets a_{\mathrm{start}}$ to $a_{\mathrm{end}}$}  
    \State $\mathbf{x_a} \gets$ maximize $H(\mathbf{x}, a)$ based on the best points in $\mathcal{X}_{a}$
    \State $\mathcal{X}_{a} \gets \{\mathcal{X}_{a}, \mathbf{x}_{a} \}$
    \EndFor
    \State \textbf{return} $\mathbf{x_a}$
\EndProcedure
\end{algorithmic}
\end{algorithm}

\section{ADDITIONAL EXPERIMENTAL STUDIES}\label{appendix:exp}
\subsection{Ranking sourcing system simulation}
\label{sec:sourcing_setup}
In the sourcing simulation experiment in Section~\ref{sec:experiments}, the recommender sourcing system has $25$ content sources and $1000$ possible pieces of content (i.e., \emph{items}) for retrieval.
We consider a $25$-dimensional retrieval policy $\mathbf{x}$ over the integer domain $[0, 50]^{25}$.
We take inspiration from the Latent Dirichlet Allocation (LDA) model~\citep{blei2003latent} in defining a generative probabilistic model of items recommended by each source.
We assume there are $8$ latent topics and that each item can be represented as a mixture over topics.
Each source contains a mixture over a set of topics, and particular items will be more likely to be recommended by topically related sources.
Such topical overlaps can create redundancy of recommendations across sources.
Retrieving more items from additional sources comes at a cost,
making sparse retrieval policies preferred.

\FloatBarrier
Before describing the simulation in pseudo-code, we need the following definitions:
\begin{itemize}
    \item $T$ is the number of latent topics.
    \item $K$ is the number of distinct items.
    \item $S$ is the number of content sources.
    \item $\theta_s \in \Delta^T$ is the topic distribution for source $s$, where $\Delta^T$ denotes the $T$-dimensional simplex. $\{\theta_s\}_{s=1}^S$ follow a Dirichlet distribution, i.e., $\theta_s \sim \textrm{Dir}(\alpha)$ where $\alpha = 0.2$.
    \item $\phi_i \in \Delta^K$ is the item distribution for each topic $i$, where $\Delta^K$ denotes the $K$-dimensional simplex. $\{\phi_i\}_{i=1}^T$ also follow a Dirichlet distribution, i.e., $\phi_i \sim \textrm{Dir}(\beta)$ where $\beta = 0.5$.
    \item $z_{s, k}$ is the topic assignment for item $k$ in source $s$ and follows a multinomial distribution: $z_{s, k} \sim \text{Multi}(\theta_s)$
    \item $w_{s, k}$ is the indicator of item $k$ is retrieved from source $s$ and follows multinomial distribution: $w_{s, k} \sim \text{Multi}(\phi_{z_{s, k}})$.
    \item $Q_i$ is the relevance score of each topic $i$ and is sampled from a log-normal distribution with mean $0.25$ and standard deviation $1.5$.
    \item $m_k$ is the relevance score of each item $k$, which is derived as the weighted average across topic scores based on the item distribution over $8$ latent topics, i.e., $m_k = \sum_{i=1}^{T} \phi_{i, k} Q_i$.
    \item $c_s$ is the infrastructure cost per fetched item for source $s$. The cost $c_s$ is assumed to be positively correlated with source relevance score $q_s = \sum_{i=1}^{T} \theta_{s, i} Q_i$ and follows a Gaussian distribution with mean $\frac{q_s}{2 \sum_{s=1}^S q_s}$ and standard deviation of $0.1$.
\end{itemize}

To simulate the retrieval of one item from the source $s$, we sample a topic for an item $k$ from the multinomial $\text{Multi}(\theta_s)$, i.e., $z_{s, k} \sim \text{Multi}(\theta_s)$, and sample an item $w_{s, k} \sim \text{Multi}(\phi_{z_{s, k}})$ where $w_{s, k}$ indicates item $k$ being retrieved from source $s$.
Given the sourcing policy $\mathbf{x} \in \mathbb{R}^S$, we execute the above sampling $\mathbf{x_s}$ times for each source $s$ as described at lines~\ref{algo:item-retrieval} in Algorithm~\ref{algo: recsus_sim}, and then compute the quality score given a list of retrieved items.

The overall content relevance score is the sum of the content relevance scores after de-duplicating the retrieved content.
The infrastructure load is a sum of products of a number of retrievals and the cost per fetched item $c_s$ for each source, in which $c_s$ varies across sources and positively correlates with the source relevance score.
This setup is based on the real-world observation that sources providing higher relevance content are generally more computationally expensive.
The objective in the benchmark experiments is a weighted sum of overall content relevance and negative infrastructure load.
In the experiment, we repeat this simulation (at line~\ref{algo:sourcing}) $1000$ times for a given policy and compute the mean and standard error of the objective values, which we refer to as the \emph{quality score} in the main text.

\begin{algorithm}
\caption{Recsys Simulation}
\label{algo: recsus_sim}
\begin{algorithmic}[1]
\Procedure{Item-retrieval}{$x_s$}
    \label{algo:item-retrieval}
    \State $\vec{n}_s \gets \overrightarrow{0} \in \mathbb{R}^K$  \Comment{number of retrievals for $K$ distinct items}
    \For{$n \gets 1$ to $x_s$} \Comment{retrieve $x_s$ items}
    \State Sample a topic for an item $k$ in source $s$ i.e. $z_{s, k} \sim \text{Multi}(\theta_s)$
    \State Sample an item $w_{s, k} \sim \text{Multi}(\phi_{z_{s, k}})$
    \State $\vec{n}_s \gets \vec{n}_s + \vec{w}_s$
    \EndFor
    \State \textbf{return} $\vec{n}_s$
\EndProcedure
\vspace{1em}
\Procedure{Sourcing}{$\mathbf{x}$}
    \label{algo:sourcing}
    \State $\vec{n} \gets \overrightarrow{0} \in \mathbb{R}^K$  \Comment{number of retrievals for $K$ distinct items}
    \For{$s \gets 1$ to $S$} \Comment{retrieve items for each source $s$}
    \State $\vec{n}_s \gets$ \Call{Item-retrieval}{$x_s$}
    \State $\vec{n} \gets \{\vec{n} + \vec{n}_s\}$
    \EndFor

    \State Compute relevance score $\text{RS} = \sum_{k=1}^K \mathbbm{1}(n_k > 0)m_k$ and infrastructure cost $C = \sum_{s=1}^S c_s \times x_s$
    \State \textbf{return} quality score $Q = \text{RS} - 0.6 \times C$
\EndProcedure
\end{algorithmic}
\end{algorithm}

\subsection{Hypervolume trace plots}
\label{subsec: Supp-hv}
We evaluate optimization performance by showing the average best obtained hypervolume across 20 replicates, with 95\% confidence interval over 100 trials. The results are shown for the sourcing problem (left), the SVM problem (middle) and the Hartmann6 function embedded into a 50D (right) in Figure~\ref{fig:hv_plot}. It can be seen that SEBO-$L_{0}$ (red traces) outperforms all the other methods and achieved the best hypervolume value over 100 iterations. The IR and ER methods with well selected regularization parameter values can sometimes achieve competitive results and usually outperform the methods with non-regularized acquisition functions, e.g. SAASBO.
\begin{figure*}[ht!]
    \centering
    \includegraphics[width=0.99\textwidth]{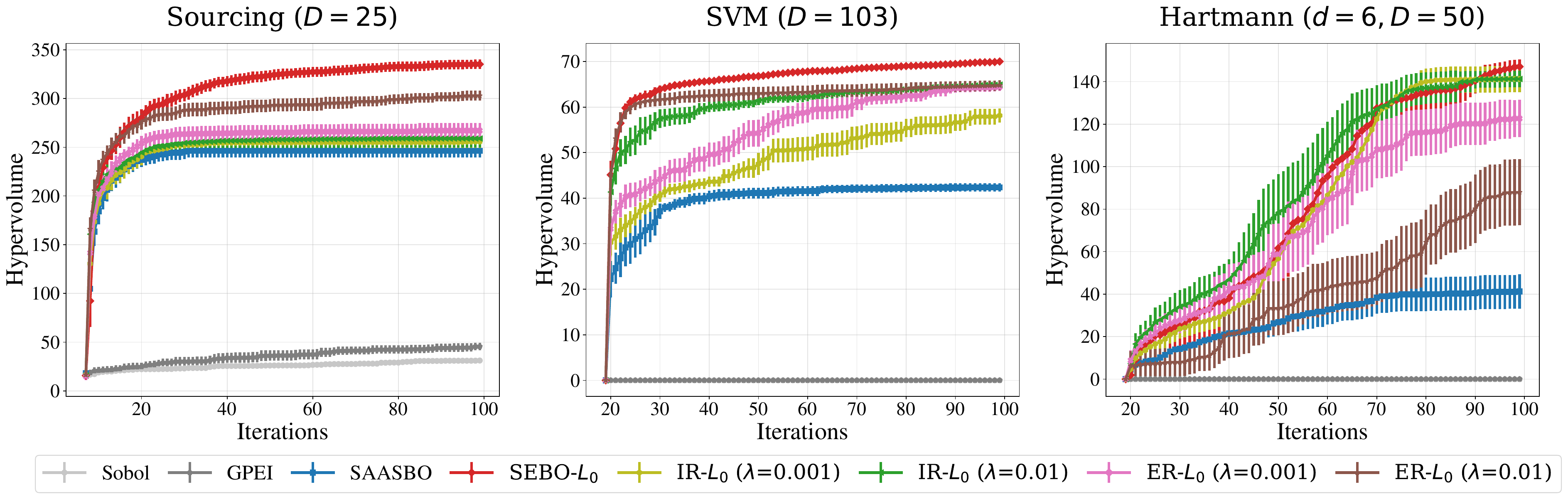}
    \caption{Hypervolume benchmark traces. (Left) Sourcing problem.(Middle) SVM problem. (Right) Hartmann6 function embedded into a 50D. The results are the average best hypervolume (with 95\% confidence interval) obtained over 100 iterations across 20 replications. SEBO-$L_{0}$, shown in red, performs the best in all three problems.}
    \label{fig:hv_plot}
\end{figure*}

\subsection{Sensitivity analysis of regularization parameter \texorpdfstring{$\lambda$}{lambda}}
\begin{figure*}[ht!]
    \centering
    \includegraphics[width=0.85\textwidth]{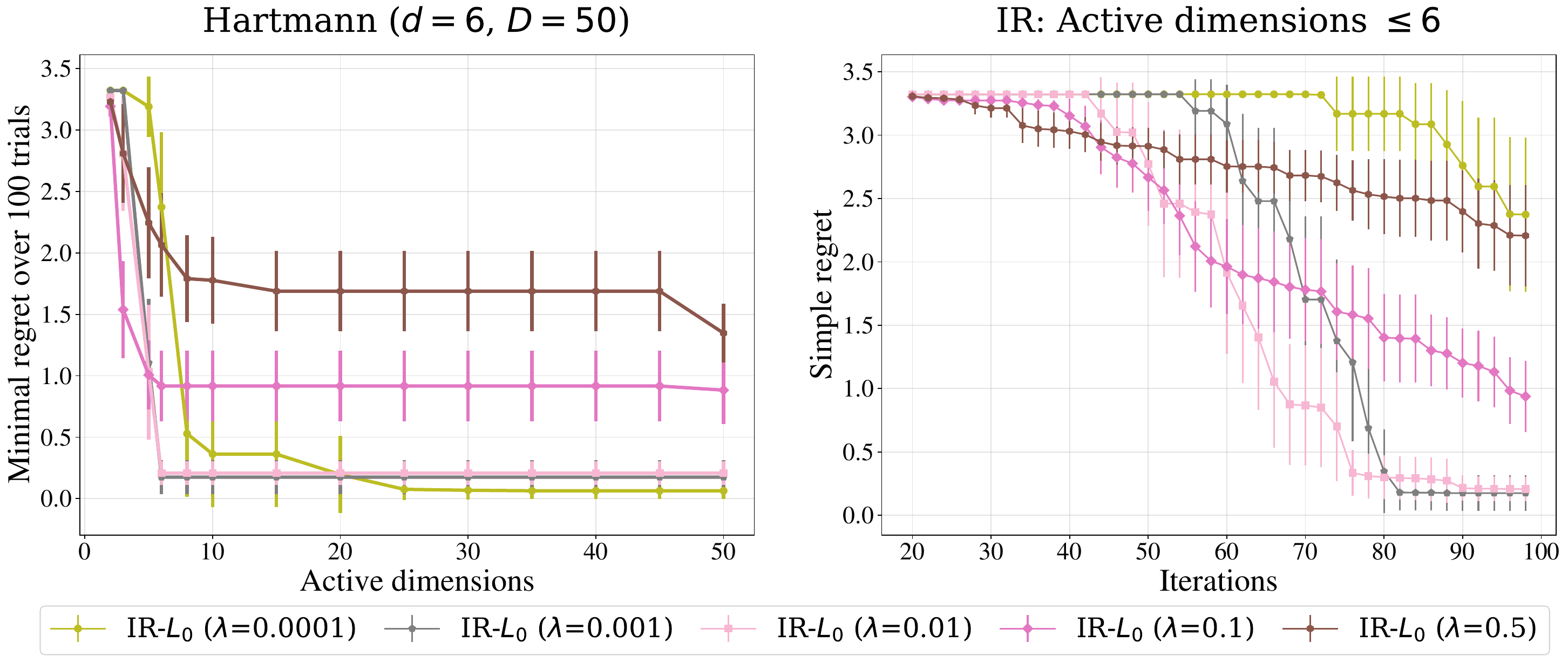}
    \caption{Results of IR with different $\lambda$ values for Hartmann6 function embedded into a $50$D space.
    (Left) The objective-sparsity trade-off after all $100$ iterations.
    (Right) The simple regret considering only observations with at most $6$ active (non-sparse) parameters.}
    \label{fig:hartmann6_ir_sweep}
\end{figure*}

\begin{figure*}[ht!]
    \centering
    \includegraphics[width=0.85\textwidth]{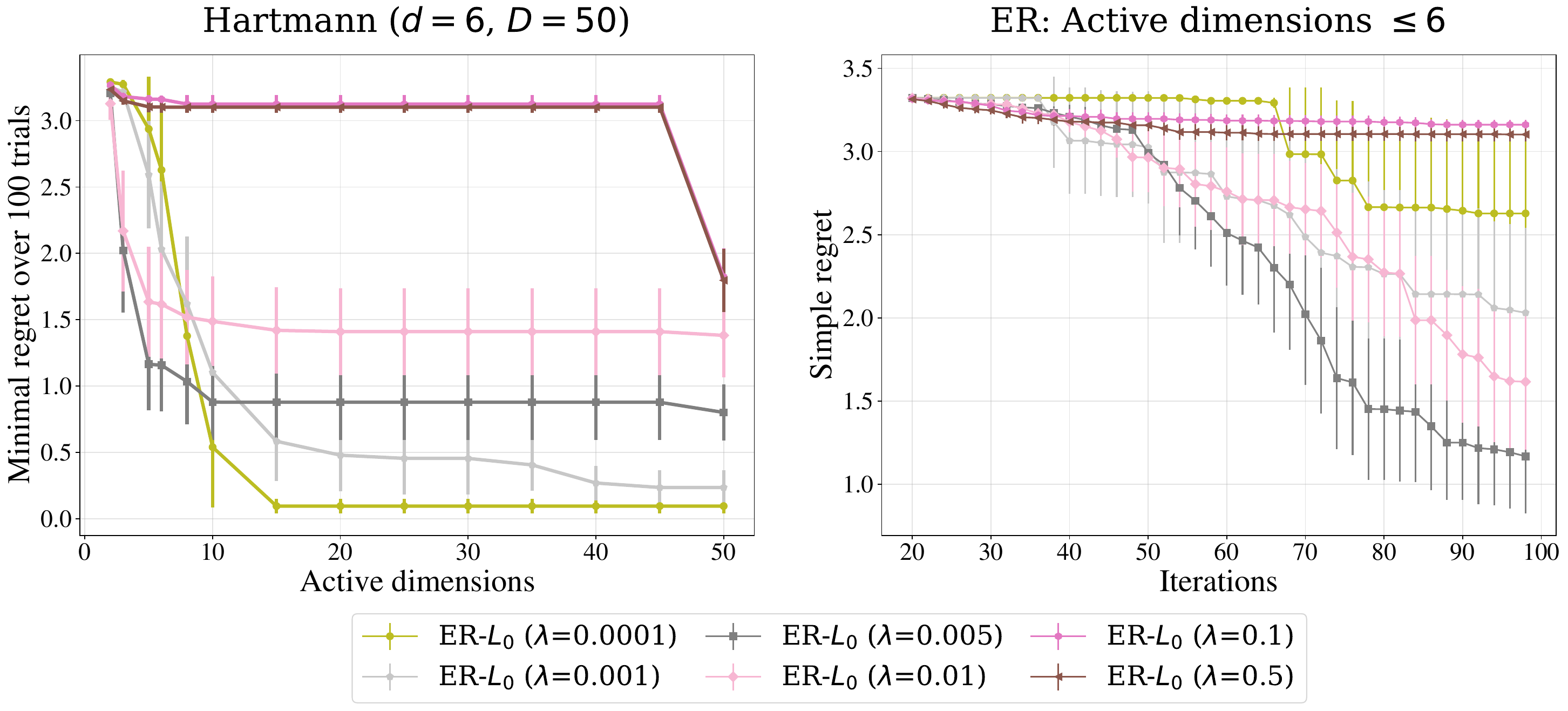}
    \caption{Results of ER with different $\lambda$ values for Hartmann6 function embedded into a $50$D space.
    (Left) The objective-sparsity trade-off after all $100$ iterations.
    (Right) The simple regret considering only observations with at most $6$ active (non-sparse) parameters.}
    \label{fig:hartmann6_er_sweep}
\end{figure*}

We conduct a sensitivity analysis of regularization parameter $\lambda$ used by IR and ER by sweeping different values of $\lambda$ on the $50$D Hartmann6 benchmark.
The results are given in Fig.~\ref{fig:hartmann6_ir_sweep} and Fig.~\ref{fig:hartmann6_er_sweep}.
We observe that we are able to control the sparsity level by appropriately choosing $\lambda$.
In general, larger $\lambda$ implies stronger regularization and results in finding configurations with a higher level of sparsity.
When $\lambda$ increases above a certain point, the regularization becomes too strong and fails to help find high-quality sparse points.

By comparing results of IR and ER for different $\lambda$ values, we note that IR is able to achieve effective optimization performance over a wider range of $\lambda$'s while ER is more sensitive to the value of $\lambda$.
This validates the discussion about ER in Section~\ref{sec:external} that ER is not as effective as IR due to ER's inability to select a new sparse point that improves over sparse points from previous iterates if the new sparse point does not improve on the dense points that are already observed.

\subsection{Benchmarks with \texorpdfstring{$L_1$}{L1} regularization}
Our proposed method can work together with different forms of sparsity. Here we show the results of ER, IR and SEBO using $L_0$ or $L_1$ regularization for the Hartmann6 function embedded in a $50$D space. As can been seen in Fig.~\ref{fig:synthetic_l1}, using $L_0$ leads to significant improvement over $L_1$ for all three methods.

\begin{figure*}[!ht]
    \centering
    \includegraphics[width=0.85\textwidth]{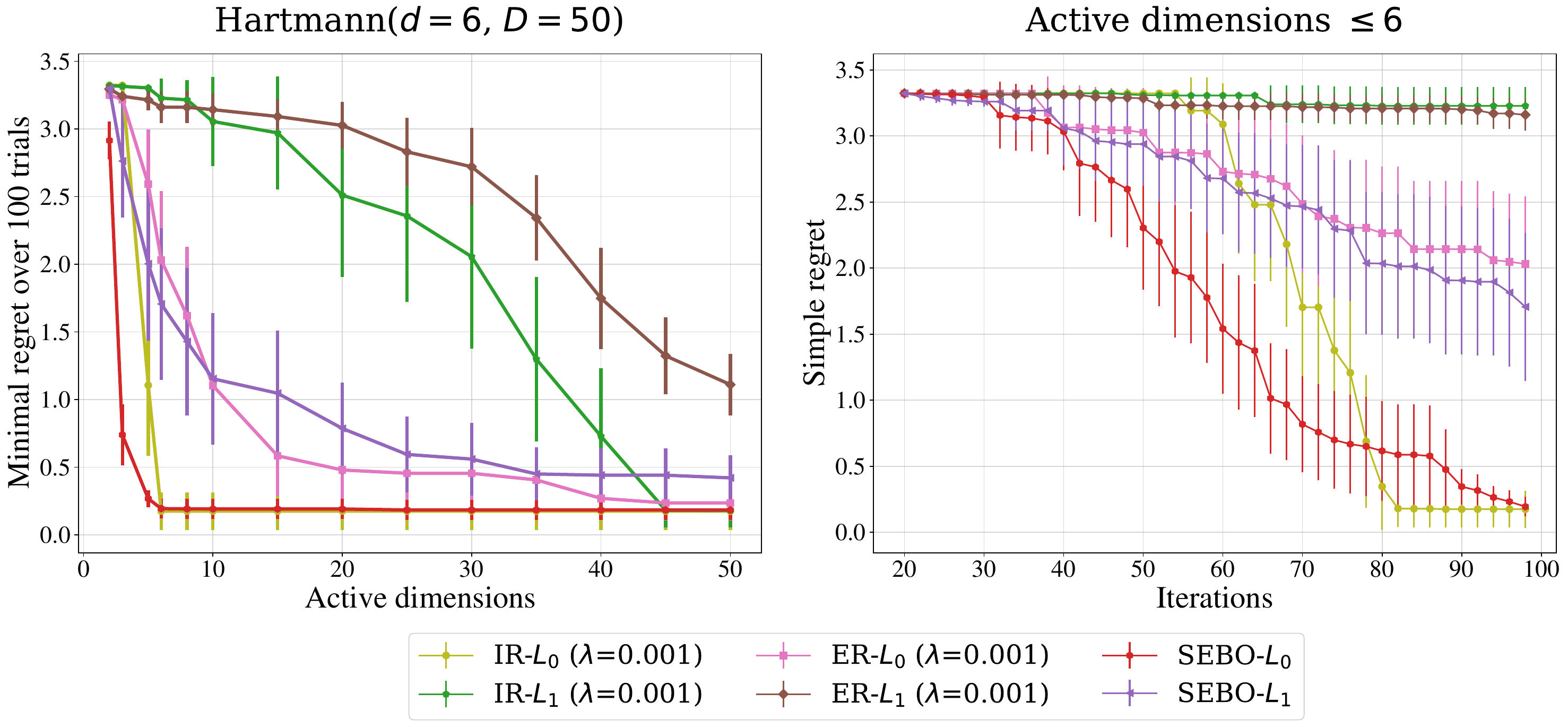}
    \caption{Results for the Hartmann6 function embedded in a $50$D space. (Left) $L_0$ regularization outperforms $L_1$ regularization in exploring the objective-sparsity trade-offs for IR, ER and SEBO. (Right) $L_0$ regularization obtains better optimization performances considering only observations with at most $6$ active (non-sparse) parameters.}
    \label{fig:synthetic_l1}
\end{figure*}

\subsection{Sensitivity Analysis of \texorpdfstring{$a_\mathrm{start}$}{a\_start} in SEBO-\texorpdfstring{$L_{0}$ }{L0} optimization}
\label{apdx:a_ablation}
\begin{figure*}[ht!]
    \centering
    \includegraphics[width=0.85\textwidth]{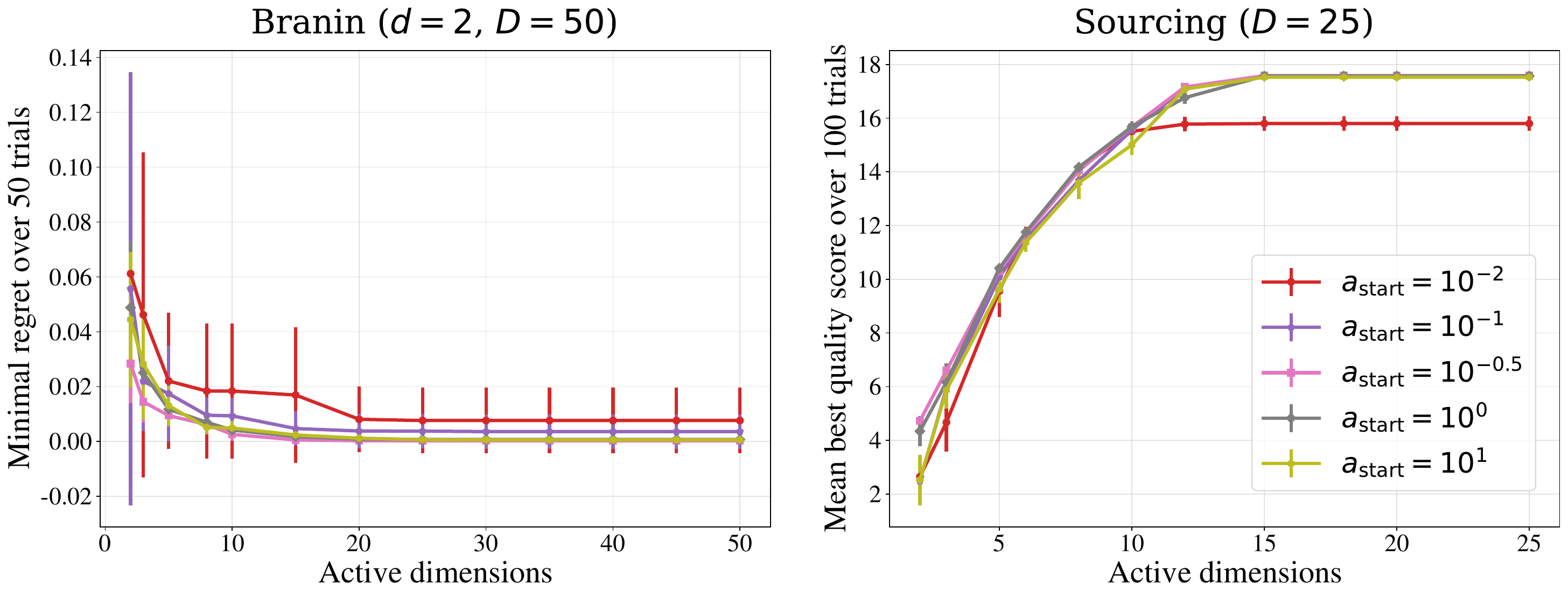}
    \caption{Ablation study of $a_{\mathrm{start}}$ in SEBO-$L_{0}$. (Left). Results of Branin $(d=2, D=50)$. (Right). Results of Sourcing $(D=25)$. There is no statistically significant difference between using different $a_{\mathrm{start}}$ except for the extremely small $a_{\mathrm{start}}$ ($=10^{-2}$). This shows the robustness of having a default $a_{\mathrm{start}}$ for optimizing SEBO-$L_{0}$ acquisition function.}
    \label{fig:ablation_ell_starting}
\end{figure*}

The value of $a_\mathrm{start}$ is set to be $10^{-0.5}$ for all the experiments.
To better understand the robustness of this choice we conducted an ablation study on the Branin$(d=2, D=50)$ and Sourcing $(D=25)$ problems considered in Section~\ref{sec:experiments}.
The results in Figure~\ref{fig:ablation_ell_starting} show that there is no statistically significant difference between using $10^{-1}$, $10^{-0.5}$, $10^{0}$ and $10^{1}$ as the value of $a_\mathrm{start}$.
However, using a value of $10^{-2}$ leads to a clear drop in performance as this starting value is too small to optimize the acquisition function.

\subsection{Ablation study on using SAAS}
\label{sec: ablation saas}
\begin{figure*}[!ht]
    \centering
    \includegraphics[width=0.85\textwidth]{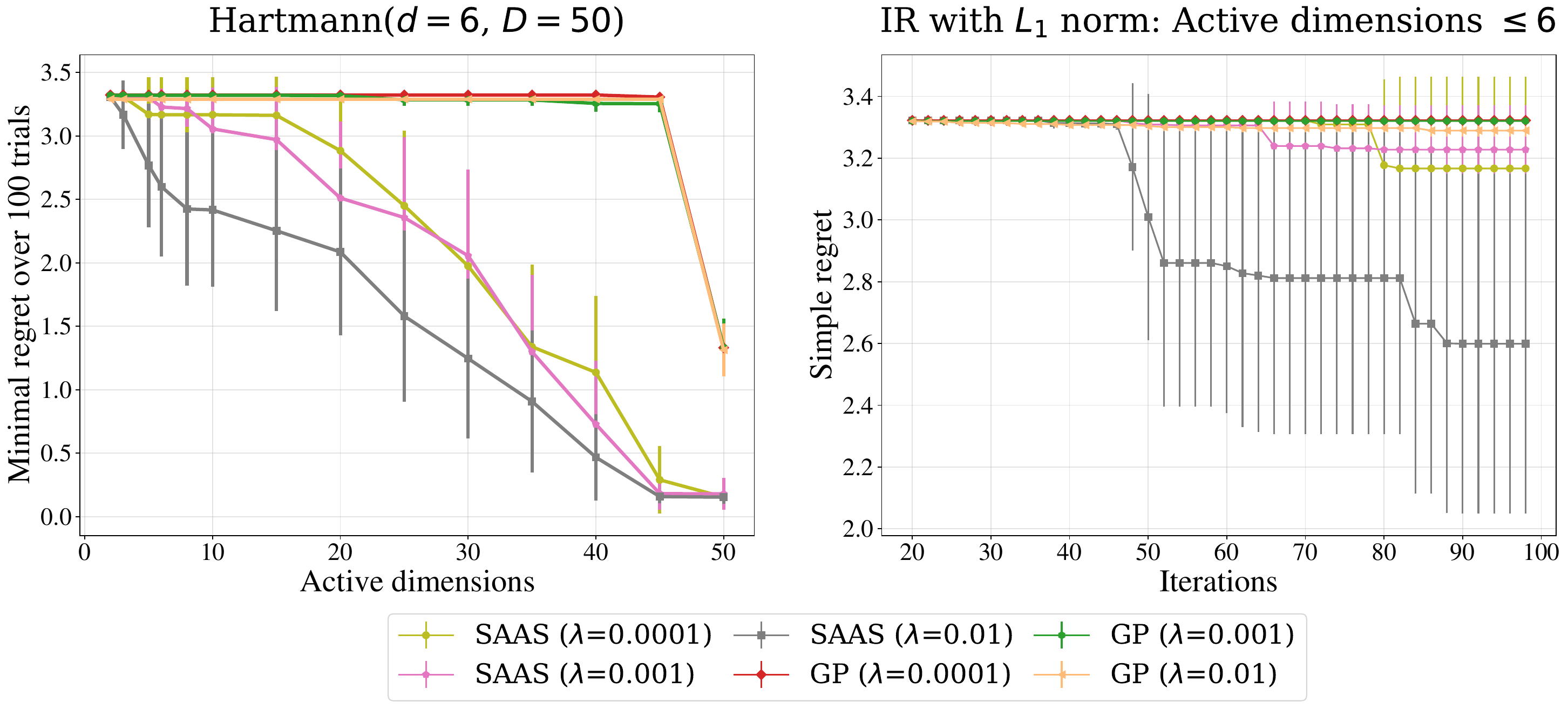}
    \caption{Results for the Hartmann6 function embedded in a $50$D space.
    IR-$L_1$ using the SAAS model significantly outperforms IR-$L_1$ using a standard GP.}
    \label{fig:compare_gp}
\end{figure*}

To illustrate the importance of using the SAAS model, we compare to using IR-$L_1$ with a standard GP in Fig.~\ref{fig:compare_gp}.
We observe that IR-$L_1$ with a standard GP fails to discover non-trivial sparse configurations for all values of $\lambda$.
This confirms that sparsity in the GP model is crucial for finding sparse configurations.
This can also be observed by comparing performances of SAASBO and GPEI in Fig.~\ref{fig:synthetic} where there is a huge gap in terms of the best function value optimized even when looking at dense points (active dimensions $=50$).

\subsection{Benchmark with additional HDBO methods}
\label{sec: hdbo comp}
We conduct evaluations of additional high-dimensional BO methods for the Hartmann6 function embedded in a $50$D space, including trust region BO (TuRBO) by~\citep{eriksson2019turbo} and Random Embedding BO (REMBO) by ~\citep{wang2013bayesian}.
The left plot in Figure~\ref{fig:compare_hdbo} shows the trade-off between the objective and sparsity after all $100$ iterations.
Although SAASBO and TuRBO achieve good non-sparse solutions, they fail to obtain sparse solutions.
REMBO does not obtain better sparse solution than SAASBO.
In the right plot, we show the simple regret considering only observations with at most $35$ active (non-sparse) parameters.
SEBO-$L_0$ outperforms these high-dimensional BO since these methods do not encourage sparse solutions.

\begin{figure*}[!ht]
    \centering
    \includegraphics[width=0.825\textwidth]{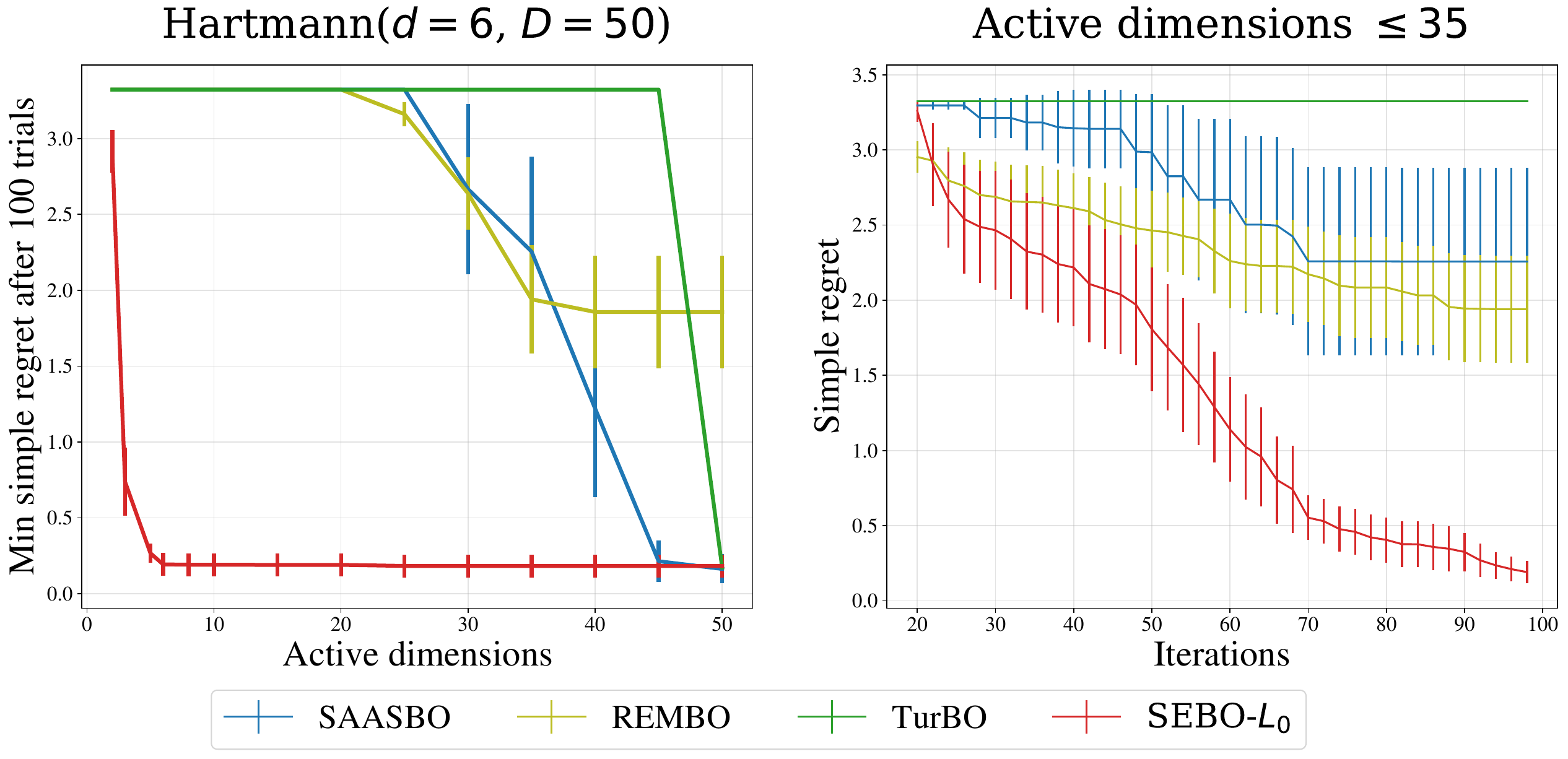}
    \caption{Results of additional high-dimensional BO methods for the Hartmann6 function embedded in a $50$D space.
    (Left) The objective-sparsity trade-off after all $100$ iterations.
    SAASBO and TuRBO, although obtaining competitive objective values with $50$ active parameters, do not encourage sparse solutions.
    (Right) The simple regret for Hartmann6 function considering only observations with at most $35$ active (non-sparse) parameters.}
    \label{fig:compare_hdbo}
\end{figure*}

\subsection{Interpretation of Sparse Solutions}
\label{sec: sparse_sol}
\paragraph{Ranking sourcing system simulation.} We examine what active dimensions are selected in the recommender sourcing system problem to understand the obtained sparse solutions.
For SEBO-$L_0$ results across $20$ replications, we obtain the optimal $25$-dimensional retrieval policy and also compute the average of retrievals per source at each sparsity level.
For each source, we compute a source quality scores based on the simulation setup stated in ~\ref{sec:sourcing_setup}.
Each source contains a mixture over a set of topics with source relevance score being $q_s$ and the infrastructure cost per fetched item being $c_s$. With this, we define and compute the source quality score as $q_s - 4 \times c_s$.
Note the score is computed for each source in order to interpret the obtained solutions and differ from the quality score used in the optimization.

In Figure~\ref{fig:sparse_sol_viz}, the left heatmap visualizes the optimal policy at different sparsity levels across $20$ replications and the middle one visualizes the average retrieval policy values. Each column corresponds to one source and is sorted based on source quality score in an ascending order (from left to right); each row represents the sparsity level (number of active dimensions). The color indicates the parameter values. As it can be seen, sources with low quality scores are turned off (zero query) and sources with higher scores have higher number of retrievals even with smaller active dimensions. This indicates that the sparse policy obtained from SEBO identifies the most effective sources at each sparsity level. The right plot in Figure~\ref{fig:sparse_sol_viz} shows the relationship between number of items retrieved from each source and source quality score with $5$ active parameters. Each dot represents a source. The curve is a fitted spline to visualize the relationship.
From both plots we can see that more items are retrieved from higher quality sources, while the number of items from lower quality sources are driven to zero.

\begin{figure*}[!ht]
    \centering
    \includegraphics[width=0.975\textwidth]{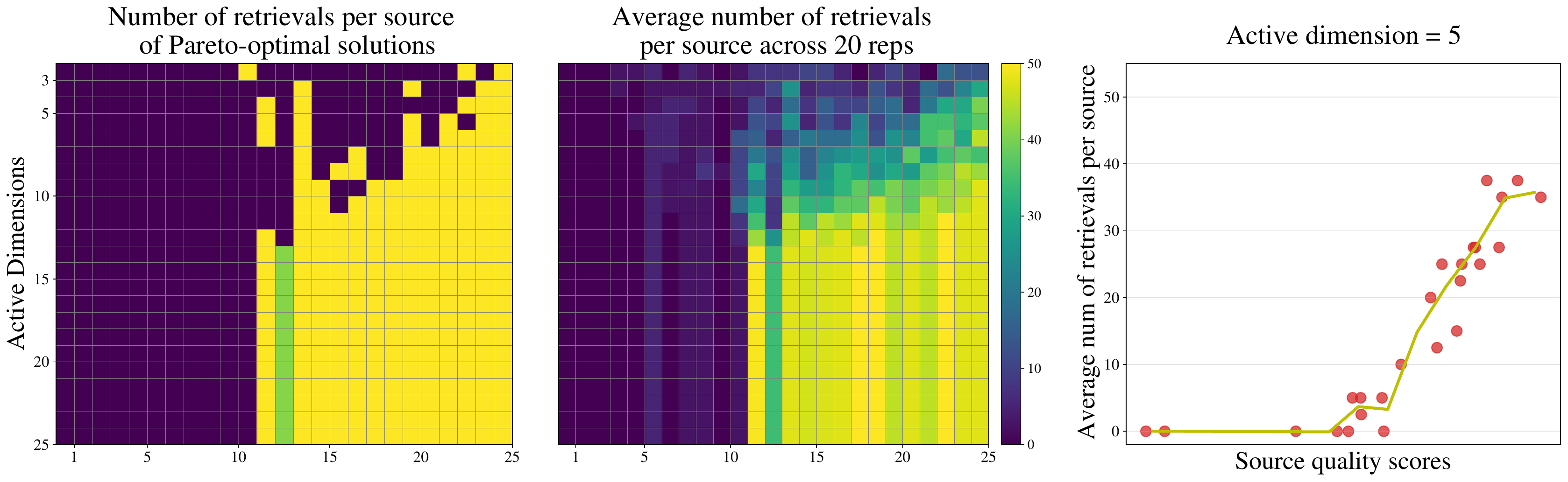}
    \caption{(Left). The heatmap of optimal retrieval policy at different sparsity levels. (Mid) The heatmap of average retrieval policy values at different sparsity levels. (Right) The scatter plot between average retrieval policy values with $5$ active parameters and source quality score. We can see that more items are retrieved from higher quality sources, while the number of items from lower quality sources are driven to zero to achieve sparsity.}
    \label{fig:sparse_sol_viz}
\end{figure*}

\paragraph{Synthetic function - Branin $(d=2, D=50)$.} Similar to Fig.~\ref{fig:ablation} (right), we compute the frequency of each parameter is turned on (non-zero) in the final Pareto frontier for each replication of SEBO-L$0$. These frequencies help us to identify the important parameters and interpret the sparse policies, shown in~Fig.~\ref{fig:interpretation} (left). The two true effective dimensions in augmented Branin $(d=2, D=50)$, colored in orange bars, have the highest frequencies and are identified by SEBO-L$0$.

\paragraph{SVM Machine learning hyperparameter tuning.} Fig.~\ref{fig:interpretation} (right) visualizes the frequency of parameter values being non-sparse in the final Pareto frontier for each replication of SEBO-L$0$. The sparse values are the center of each interval of the three hyperparameters $\gamma$, $C$ and $\epsilon$. For the augmented parameters, values being zero are considered as sparse. The three orange bars correspond to the three effective hyperparameters of the SVM, which obtains high frequencies of being non-sparse. The gray bars, corresponding to the augmented dimensions, have much lower frequencies.

\begin{figure*}[!ht]
    \centering
    \includegraphics[width=0.8\textwidth]{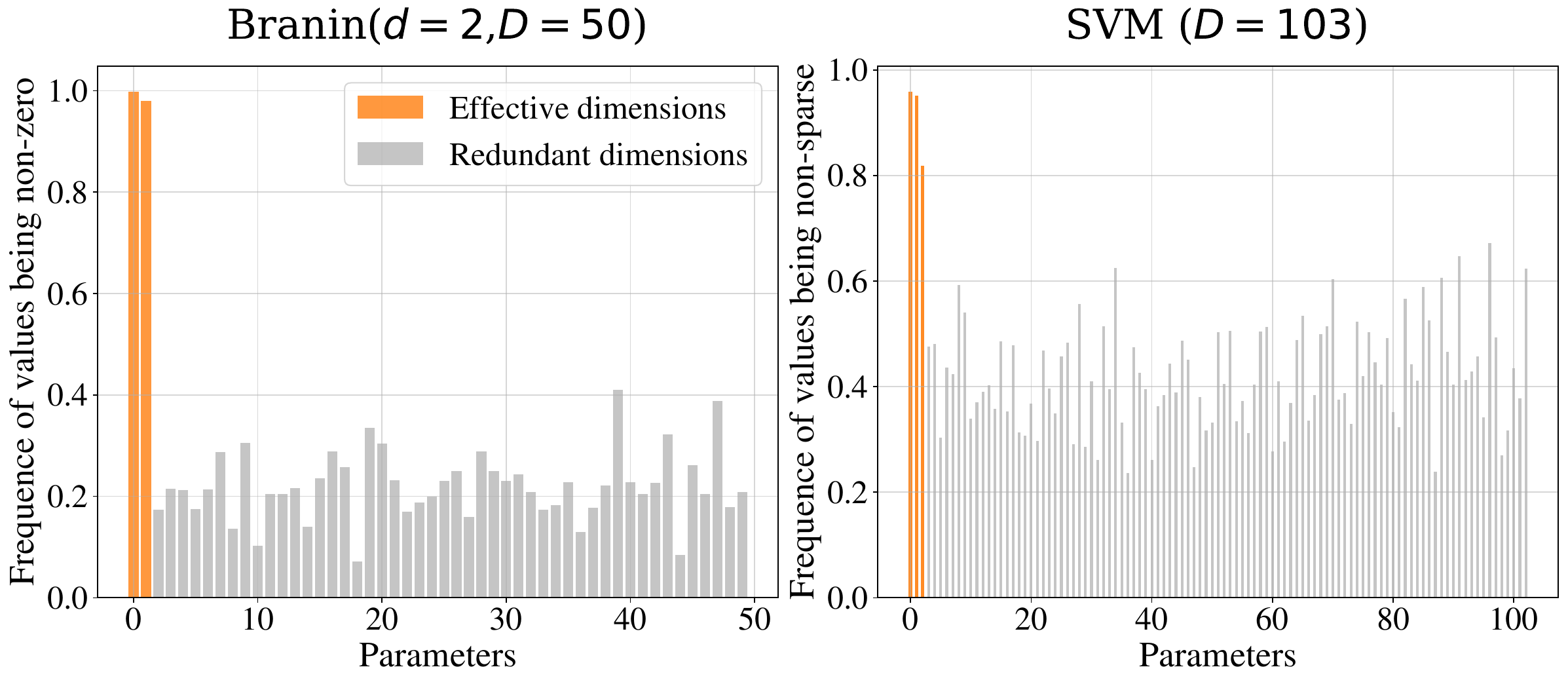}
    \caption{(Left) Branin $(d=2, D=50)$. The $2$ true effective parameters (colored as orange) are more frequently set to be non-zero in Pareto optimal configurations. (Right) SVM $(d=3, D=103)$. The $3$ effective hyperparameters (orange) of the SVM have higher frequencies of being non-sparse compared with the augmented dimensions (gray bars).}
    \label{fig:interpretation}
\end{figure*}

\subsection{Low-dimensional BO Problem}
SEBO can also be applied to low-dim problems using arbitrary GP models as it targets the trade-off between objective and sparsity. In the experiments section (Section~\ref{sec:experiments}), we focus on high-dimensional problems because sparsity (and interpretability) tends to be more important with more parameters. In Figure~\ref{fig:hartmann6_PF}, we compare the performance on the Hartmann6 problem for Sobol, SAASBO, GPEI, and SEBO with $L_0$ or $L_1$ penalty using a standard GP as a surrogate model. This problem is known to have structures where some dimensions are more important than others for maximizing function value.
SEBO-$L_0$ with a standard GP achieves the best trade-off in the low-dimensional (6D) problem.

\begin{figure}[!ht]
    \centering
    \includegraphics[width=0.45\textwidth]{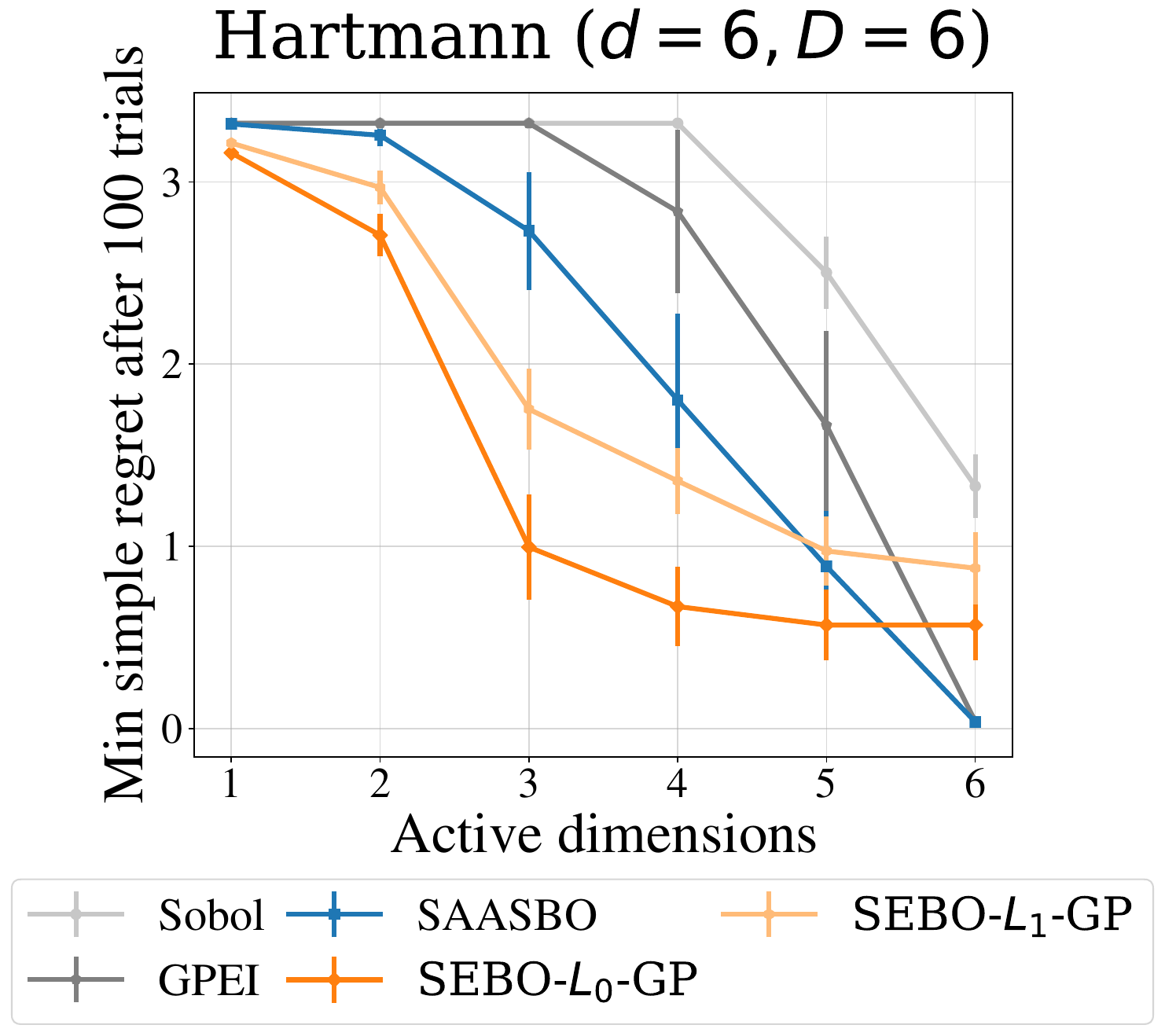}
    \caption{
        Hartmann6 function where 0 is considered sparse. Standard GP (without SAAS model) is used as the GP surrogate model for SEBO-$L_1$ and SEBO-$L_0$.
    }
    \label{fig:hartmann6_PF}
\end{figure}

\thispagestyle{empty}

\end{document}